\documentclass{article}

\def\isfinal{1}

\PassOptionsToPackage{numbers}{natbib}

\if \isfinal1
\usepackage[final]{neurips_2020}
\else
\usepackage{neurips_2020}
\fi

\usepackage[utf8]{inputenc} 
\usepackage[T1]{fontenc}    
\usepackage{amsthm, thmtools}
\usepackage[hyphens]{url}            
\usepackage{booktabs}       
\usepackage{nicefrac}       
\usepackage{microtype}      

\usepackage[export]{adjustbox}
\usepackage{commath}
\usepackage[title]{appendix}
\usepackage{amssymb}
\usepackage{thm-restate}
\usepackage{textcomp}
\usepackage[lofdepth,lotdepth]{subfig}
\usepackage[utf8]{inputenc}
\usepackage{pgf, xcolor}
\usepackage[shortlabels]{enumitem}
\usepackage{silence}
\usepackage{mathtools}
\usepackage{epigraph}
\usepackage[ruled]{algorithm2e}
\usepackage{pifont}
\usepackage{wrapfig}
\usepackage{adjustbox}
\usepackage{array}

\usepackage{longtable}
\usepackage[tableposition=below]{caption}
\usepackage{amsmath}
\usepackage{hyperref,cleveref} 
\hypersetup{
    colorlinks,
    linkcolor={red!50!black},
    citecolor={blue!50!black},
    urlcolor={blue!80!black}
}

\newcommand{\inlinesprite}[2]{
  \begingroup\normalfont\hspace{-0.5em}
  \raisebox{-.1\height}{\includegraphics[height=0.65em, clip, trim={#1}]{sprites/sprites-#2.pdf}}
  \hspace{-0.22em}\endgroup
}
\DeclareRobustCommand {\agentsprite}[1][gray]    {\inlinesprite{2cm 4cm 2cm 0cm}{#1}}
\DeclareRobustCommand {\lifesprite}[1][gray]     {\inlinesprite{0cm 3cm 4cm 1cm}{#1}}

\DeclareRobustCommand {\treesprite}[1][gray]     {\inlinesprite{3cm 1cm 1cm 3cm}{#1}}

\DeclareRobustCommand {\wallsprite}[1][gray]     {\inlinesprite{2cm 2cm 2cm 2cm}{#1}}
\DeclareRobustCommand {\cratesprite}[1][gray]     {\inlinesprite{3cm 2cm 1cm 2cm}{#1}}

\DeclareRobustCommand {\spawnsprite}[1][gray]    {\inlinesprite{1cm 1cm 3cm 3cm}{#1}}
\DeclareRobustCommand {\exitsprite}[1][gray]     {\inlinesprite{2cm 1cm 2cm 3cm}{#1}}

\newcommand*{\naive}{\texttt{Naive}}
\newcommand*{\aup}{\texttt{AUP}}
\newcommand*{\aupp}{\aup$_\text{proj}$}
\newcommand*{\dqn}{\texttt{DQN}}
\newcommand*{\ppo}{\texttt{PPO}}

\newtheorem{thm}{Theorem}
\newtheorem{prop}[thm]{Proposition}

\theoremstyle{definition}
\newtheorem{definition}{Definition}

\theoremstyle{remark}

\renewcommand{\abs}[1]{\left|#1\right|}

\renewcommand{\set}[1]{\left\{#1\right\}}
\newcommand{\Vf}[2][R]{V^{*}_{#1} \left(#2\right)}

\newcommand{\St}{\mathcal{S}}

\newcommand{\R}{\mathcal{R}}

\newcommand{\geom}[1][1]{\frac{#1}{1-\gamma}}

\title{Avoiding Side Effects in Complex Environments}

\author{
\textbf{Alexander Matt Turner\thanks{Equal contribution.} \hspace{25pt} Neale Ratzlaff$^*$
 \hspace{30pt} Prasad Tadepalli \vspace{0pt}} \\
Oregon State University\\
\texttt{\{turneale@, ratzlafn@, tadepall@eecs.\}oregonstate.edu}
}

\begin{document}

\maketitle

\setcounter{footnote}{0} 

\begin{abstract}
Reward function specification can be difficult. Rewarding the agent for making a widget may be easy, but penalizing the multitude of possible negative side effects is hard. In toy environments,  Attainable Utility Preservation (\textsc{aup}) avoided side effects by penalizing shifts in the ability to achieve randomly generated goals \citep{turner2020conservative}. We scale this approach to large, randomly generated environments based on Conway's Game of Life. By preserving optimal value for a single randomly generated reward function, \textsc{aup} incurs modest overhead while leading the agent to complete the specified task and avoid many side effects. Videos and code are available at \href{https://avoiding-side-effects.github.io/}{https://avoiding-side-effects.github.io/}.
\end{abstract} 

\section{Introduction}
Reward function specification can be difficult, even when the desired behavior seems clear-cut.  For example, rewarding  progress in a race led a reinforcement learning (\textsc{rl}) agent to collect checkpoint reward, instead of completing the race \citep{Victoria_specification}. We want to minimize the negative side effects of misspecification: from a robot which breaks equipment, to content recommenders which radicalize their users, to potential future AI systems which negatively transform the world \citep{bostrom_superintelligence_2014,russell_human_2019}. 

Side effect avoidance poses a version of the ``frame problem'': each action can have many effects, and it is impractical to explicitly penalize all of the bad ones \citep{brown2014frame}. For example, a housekeeping agent should clean a dining room without radically rearranging furniture, and a manufacturing agent should assemble widgets without breaking equipment. A general, transferable solution to side effect avoidance would ease reward specification: the agent's designers could just positively specify what should be done, as opposed to negatively specifying what should not be done. 

Breaking equipment is bad because it hampers future optimization of the intended ``true'' objective (which includes our preferences about the factory). That is, there often exists a reward function $R_{\text{true}}$ which fully specifies the agent's task within its deployment context. In the factory setting, $R_\text{true}$ might encode ``assemble widgets, but don't spill the paint, break the conveyor belt, injure workers, etc.''

We want the agent to preserve optimal value for this true reward function. While we can accept  suboptimal actions (e.g.\ pacing the factory floor), we cannot accept the destruction of value for the true task. By avoiding negative side effects which decrease value for the true task, the designers can correct any misspecification and eventually achieve low regret for $R_\text{true}$.

\paragraph{Contributions.} 
Despite being unable to directly specify $R_\text{true}$, we demonstrate a method for preserving its optimal value anyways. \citet{turner2020conservative} introduced \textsc{aup}; in their toy environments,  preserving  optimal value for many randomly generated reward functions often preserves the optimal value for $R_\text{true}$. In this paper, we generalize \textsc{aup} to combinatorially complex environments and evaluate it on four tasks from the chaotic and challenging SafeLife test suite \citep{wainwright2019safelife}. We show the rather surprising result that by preserving optimal value for a \emph{single} randomly generated reward function, \textsc{aup} preserves optimal value for $R_\text{true}$ and thereby avoids negative side effects. 


\section{Prior Work}
\textsc{Aup} avoids negative side effects in small gridworld environments while preserving optimal value for uniformly randomly generated auxiliary reward functions \citep{turner2020conservative}. While \citet{turner2020conservative} required many  auxiliary reward functions in their toy environments, we show that a single auxiliary reward function -- learned unsupervised -- induces competitive performance and discourages side effects in complex environments.

Penalizing decrease in (discounted) state reachability achieves similar results \citep{relative_original}. However, this approach has difficulty scaling: naively estimating all reachability functions is a task quadratic in the size of the state space. In \cref{sec:theory}, \cref{mut-bound} shows that preserving  the reachability of the initial state \citep{eysenbach2018leave} bounds the maximum decrease in optimal value for $R_\text{true}$. Unfortunately, due to  irreversible dynamics, initial state reachability often cannot be preserved.

\citet{shah2018the} exploit information contained in the initial state of the environment to infer which side effects are negative; for example, if vases are present, humans must have gone out of their way to avoid them, so the agent should as well. In the multi-agent setting, empathic deep Q-learning preserves  optimal value for another agent in the environment \citep{bussmann2019towards}. We neither assume nor model the presence of another agent.

Robust optimization selects a trajectory which maximizes the minimum return achieved under a feasible set of reward functions \citep{regan2010robust}. However, we do not assume we can specify the feasible set. In constrained \textsc{mdp}s, the agent obeys constraints while maximizing the observed reward function \citep{altman1999constrained, achiam2017constrained, zhang2018minimax}. Like specifying reward functions, exhaustively specifying constraints is  difficult. 

Safe reinforcement learning  focuses on avoiding catastrophic mistakes during training and ensuring that the learned policy satisfies certain constraints \citep{pecka2014safe,garcia2015comprehensive,berkenkamp2017safe,ray2019benchmarking}. While this work considers the safety properties of the learned policy, \textsc{aup} should be compatible with safe \textsc{rl} approaches.

We train value function networks separately, although \citet{schaul_universal_2015} demonstrate a value function predictor which generalizes across both states and goals.

\section{\textsc{AUP} Formalization}
Consider a Markov decision process (\textsc{mdp}) $\langle \mathcal{S},\mathcal{A}, T, R, \gamma \rangle$  with state space $\mathcal{S}$, action space $\mathcal{ A}$, transition function $T:\mathcal{ S}\times \mathcal{A} \to\Delta(\mathcal{S})$, reward function $R:\mathcal{ S}\times \mathcal{A}\to\mathbb{R}$, and discount factor $\gamma \in [0,1)$. We assume the agent may take a no-op action $\varnothing \in \mathcal{A}$. We refer to $\Vf{s}$ as the \emph{optimal value} or \emph{attainable utility}  of reward function $R$ at state $s$.

To define \textsc{aup}'s pseudo-reward function, the designer provides a finite reward function set $\mathcal{R}\subsetneq  \mathbb{R}^\mathcal{S}$, hereafter referred to as the \emph{auxiliary set}. This set does not necessarily contain $R_\text{true}$. Each auxiliary reward function $R_i\in \mathcal{R}$ has a learned Q-function $Q_{i}$. 

\textsc{Aup} penalizes average change in action value for the auxiliary reward functions. The motivation is that by not changing optimal value for a wide range of auxiliary reward functions, the agent may avoid decreasing optimal value for $R_{\text{true}}$.

\begin{definition}[\textsc{Aup} reward function \citep{turner2020conservative}] Let $\lambda\geq 0$. Then
\begin{equation}
\label{eq:aup}
    R_\textsc{aup}(s, a) \vcentcolon = R(s,a) - \frac{\lambda}{|\mathcal{R}|} \sum_{R_i\in \mathcal{R}} \left | Q^*_{i}(s,a) - Q^*_{i}(s, \varnothing) \right |.
\end{equation}
\end{definition}

The regularization parameter $\lambda$ controls penalty severity. In appendix \ref{sec:theory}, \cref{mut-bound} shows that \cref{eq:aup} only lightly penalizes easily reversible actions.  In practice, the learned auxiliary $Q_{i}$ is a stand-in for the optimal Q-function $Q^*_{i}$. 

\section{Quantifying Side Effect Avoidance with SafeLife}

In Conway's Game of Life, cells are alive or dead. Depending on how many live neighbors surround a cell, the cell comes to life, dies, or retains its state. Even simple initial conditions can evolve into complex and chaotic patterns, and the Game of Life  is Turing-complete when played on an infinite grid \citep{rendell2002turing}. 

SafeLife turns the Game of Life into an actual game. An autonomous agent moves freely through the world, which is a large finite grid. In the eight cells surrounding the agent, no cells spawn or die – the agent can disturb dynamic patterns  by merely approaching them. There are many colors and kinds of cells, many of which have unique effects (see \cref{fig:safelife-still}). 

\begin{figure}[ht]
\centering 
\subfloat[][\texttt{append-spawn}]{\label{append-spawn}
    \includegraphics[width=.4\linewidth]{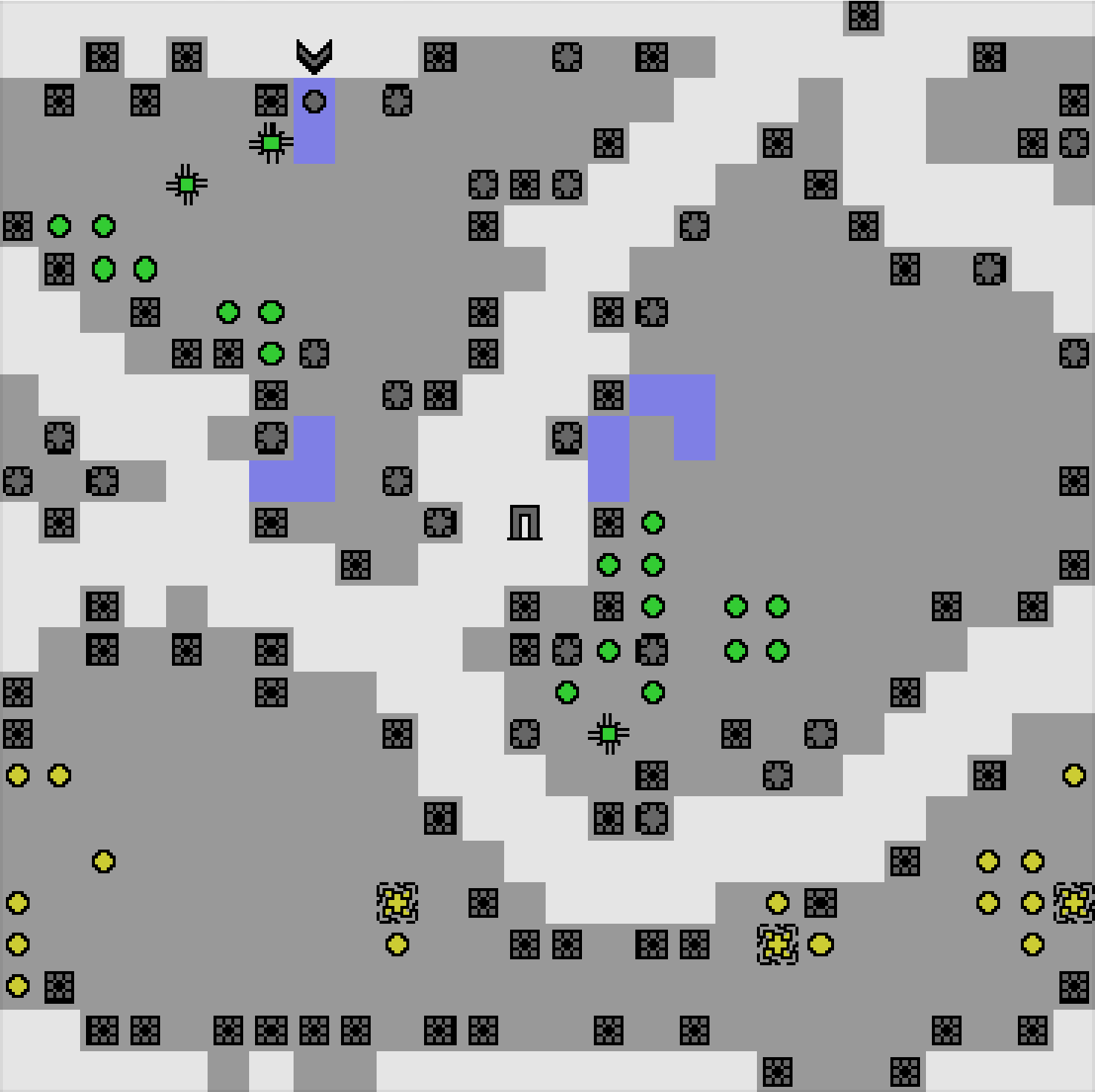}
}
\subfloat[][\texttt{prune-still-easy}]{\label{prune-still-easy}
    \includegraphics[width=.4\linewidth]{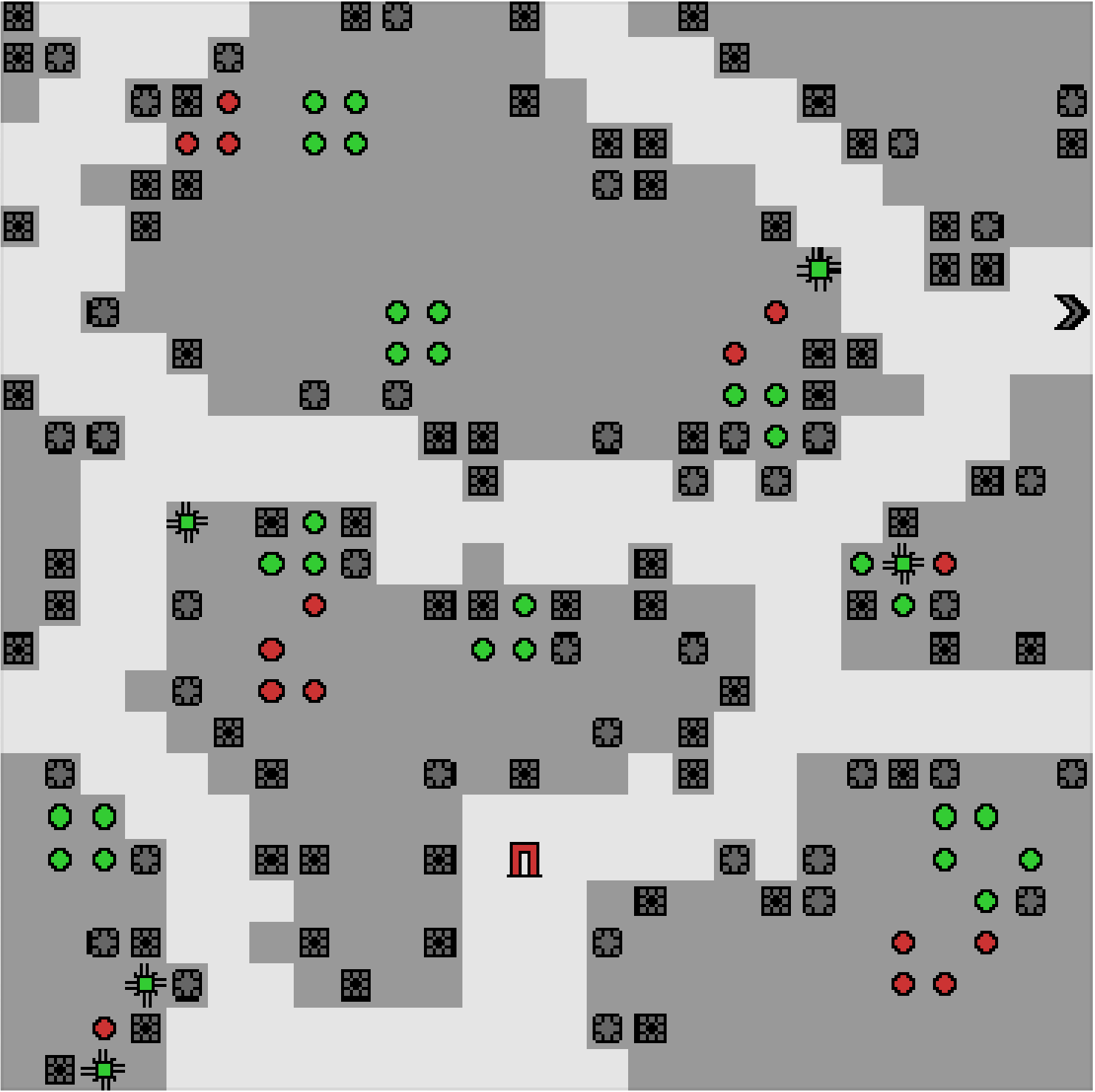}
}
\caption{Trees~\mbox{(\treesprite[green])} are permanent living cells. The agent~\mbox{(\agentsprite)} can move crates~\mbox{(\cratesprite)} but not walls~\mbox{(\wallsprite)}. The screen wraps vertically and horizontally.  \protect\subref{append-spawn}: The agent receives reward for creating gray cells~\mbox{(\lifesprite)} in the blue  areas. The goal~\mbox{(\exitsprite)} can be entered when some number of gray cells are present. Spawners~\mbox{(\spawnsprite[yellow])}  stochastically create yellow living cells. \protect\subref{prune-still-easy}: The agent receives reward for removing red cells; after some number have been removed, the goal turns red~\mbox{(\exitsprite[red])} and can be entered.\label{fig:safelife-still}}
\end{figure}

To understand the policies incentivized by \cref{eq:aup}, we now consider a simple example. \Cref{fig:example}  compares a policy which only optimizes the SafeLife reward function $R$, with an \textsc{aup} policy that also preserves the optimal value for a single auxiliary reward function ($|\mathcal{R}|=1$). 

\begin{figure}[ht]
\centering 
\subfloat[][Baseline trajectory]{\label{ex1}
    \includegraphics[width=.495\linewidth]{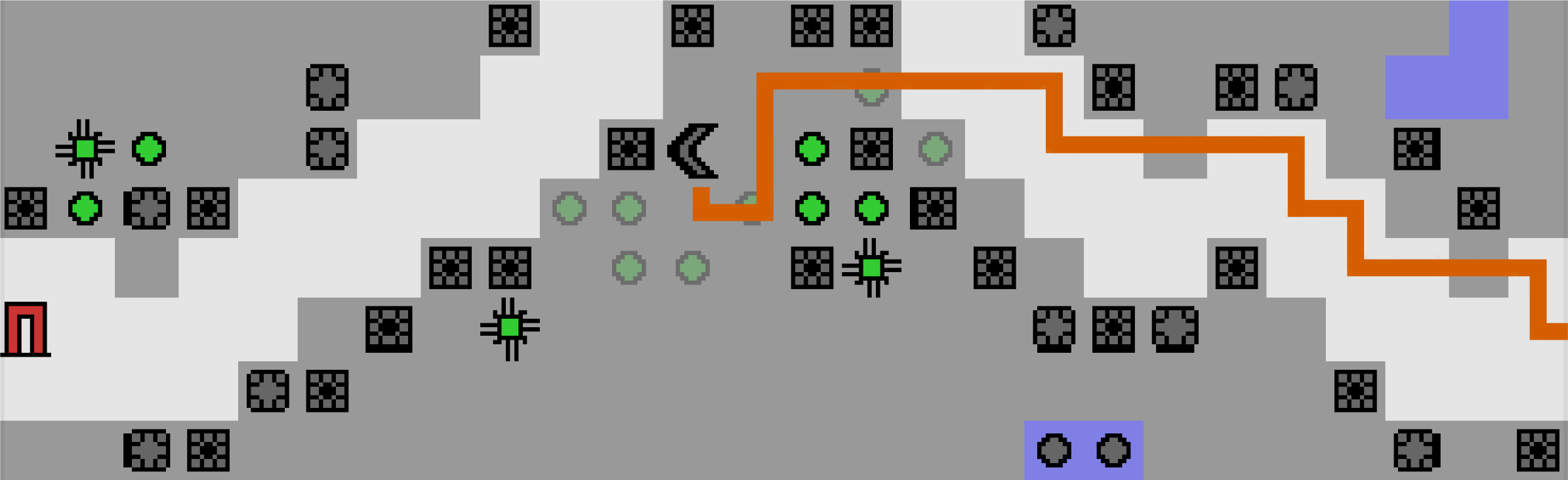}}
\subfloat[][\textsc{aup} trajectory]{\label{ex2}
    \includegraphics[width=.495\linewidth]{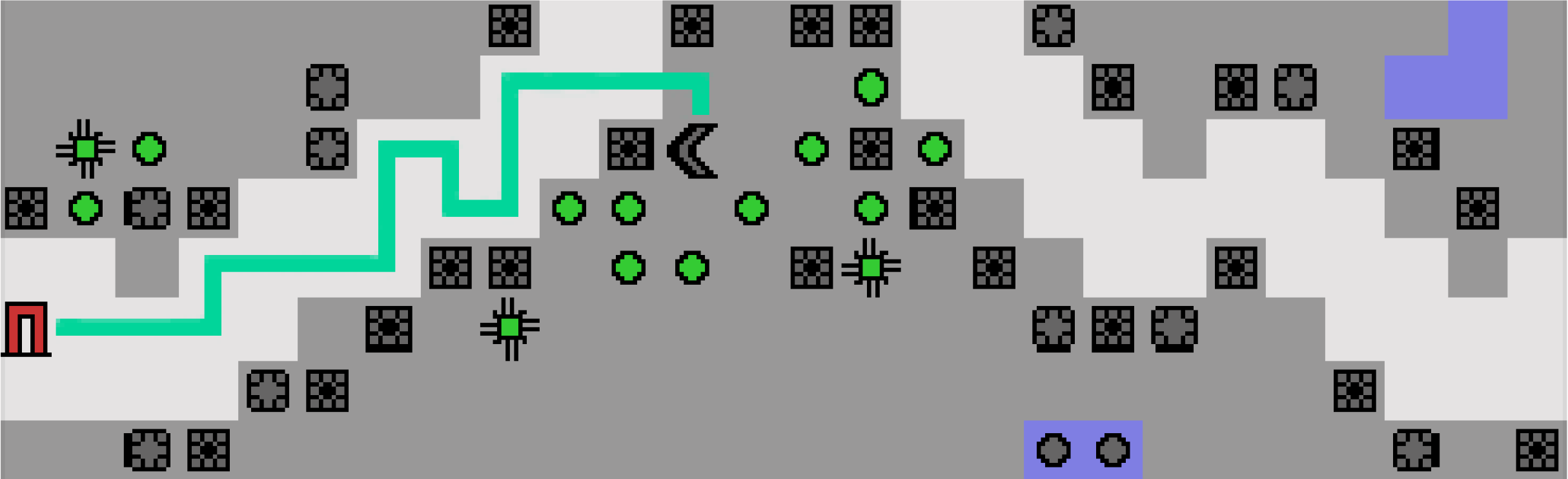}}
\caption{The agent~\mbox{(\agentsprite)} receives 1 primary reward for entering the goal~\mbox{(\exitsprite[red])}. The agent can move in the cardinal directions, destroy cells in the cardinal directions, or do nothing. Walls~\mbox{(\wallsprite)} are not movable. The right end  of the screen wraps around to the left. 
 \protect\subref{ex1}: The learned trajectory for the misspecified primary reward function $R$ destroys fragile green cells~\mbox{(\lifesprite[green])}. \protect\subref{ex2}: Starting from the same state, \textsc{aup}'s trajectory preserves the green cells.}  
    \label{fig:example}
\end{figure}

Importantly, we did not hand-select an informative auxiliary reward function in order to induce the trajectory of \cref{ex2}. Instead, the auxiliary reward was the output of a one-dimensional observation encoder, corresponding to a continuous Bernoulli variational autoencoder (\textsc{cb-vae}) \citep{loaiza2019continuous} trained through random exploration.

While  \citet{turner2020conservative}'s \textsc{aup} implementation uniformly randomly generated reward functions over the observation space, the corresponding Q-functions would have extreme sample complexity in the high-dimensional SafeLife tasks (\cref{tab:comparison}). In contrast, the \textsc{cb-vae} provides a structured and learnable auxiliary reward signal.

\begin{table}[ht]
    \centering
    \begin{tabular}{cc}
    \toprule
    AI safety gridworlds \citep{leike_ai_2017} & SafeLife \citep{wainwright2019safelife}\\
    \midrule
    Dozens of states & Billions of states \\
    Deterministic dynamics & Stochastic dynamics \\
    Handful of preset environments & Randomly generated environments \\
    One side effect per level & Many side effect opportunities\\
    Immediate side effects & Chaos unfolds over time\\
    \bottomrule
    \end{tabular}
    \vspace{8pt}
    \caption{\citet{turner2020conservative} evaluated \textsc{aup} on toy environments. In contrast, SafeLife challenges modern \textsc{rl} algorithms and is well-suited for testing side effect avoidance.}
    \vspace{-12pt}
    \label{tab:comparison}
\end{table}

\section{Experiments}
\label{sec:experiments} 

Each time step, the agent observes a $25\times 25$ grid cell window centered on its current position. The agent can move in the cardinal directions, spawn or destroy a living cell in the cardinal directions, or do nothing. 

We follow \citet{wainwright2019safelife} in scoring side effects as the degree to which the agent perturbs green cell patterns. Over an episode of $T$ time steps, side effects are quantified as the  Wasserstein 1-distance between the configuration of green cells had the state evolved naturally for $T$ time steps, and the actual configuration at the end of the episode. As the primary reward function $R$ is indifferent to green cells, this proxy measures the safety performance of learned policies. 

If the agent never disturbs green cells, it achieves a perfect score of zero; as a rule of thumb, disturbing a green cell pattern increases the score by 4. By construction, minimizing side effect score preserves $R_\text{true}$'s optimal value, since $R_\text{true}$ encodes our preferences about the existing green patterns.

\subsection{Comparison}\label{sec:line}
\paragraph{Method.}
We describe five conditions below and evaluate them on the  \texttt{append-spawn} (\cref{append-spawn}) and \texttt{prune-still-easy} (\cref{prune-still-easy}) tasks. Furthermore, we include two variants of \texttt{append-spawn}: \texttt{append-still} (no stochastic spawners and more green cells) and \texttt{append-still-easy} (no stochastic spawners). The primary, specified SafeLife  reward functions are as follows: \texttt{append-*} rewards maintaining gray cells in the blue tiles (see \cref{append-spawn}), while \texttt{prune-still-easy} rewards the agent for removing red cells (see \cref{prune-still-easy}).

For each task, we randomly generate a set of 8 environments to serve as the curriculum. On each generated curriculum, we evaluate each condition on several randomly generated seeds. The agents are evaluated on their training environments. In general, we generate 4 curricula per task; performance metrics are averaged over 5 random seeds for each curriculum. We use curriculum learning because the \textsc{ppo} algorithm seems unable to learn environments one at a time.

We have five conditions: {\ppo}, {\dqn}, {\aup}, {\aupp}, and {\naive}.\footnote{We write ``{\aup}'' to refer to the experimental condition and ``\textsc{aup}'' to refer to the method of Attainable Utility Preservation.} Excepting {\dqn}, the Proximal Policy Optimization (\textsc{ppo} \citep{schulman2017proximal}) algorithm trains each condition on a different reward signal for five million (5M) time steps. See \cref{sec:train} for architectural and training details.  

\begin{enumerate}[leftmargin=\widthof{{\aupp}}+\labelsep]
\item[{\ppo}] Trained on the primary SafeLife reward function $R$ without a side effect penalty.

\item[{\dqn}] Using \citet{mnih2015human}'s \textsc{dqn}, trained on the primary SafeLife reward function $R$ without a side effect penalty.

\item[{\aup}] For the first 100,000 (100K) time steps, the agent uniformly randomly explores to collect observation frames. These frames are used to train a continuous Bernoulli variational autoencoder with a 1-dimensional latent space and encoder network $E$. 

The auxiliary reward function is then the output of the encoder $E$; after training the encoder for the first 100K steps, we train a Q-value network for the next 1M time steps.  This learned $Q_{R_1}$ defines the $R_\textsc{AUP}$ penalty term (since $\abs{\R}=1$; see \cref{eq:aup}). 

While the agent trains on the $R_\textsc{aup}$ reward signal for the final 3.9M steps, the $Q_{R_1}$ network is fixed and $\lambda$ is linearly increased from .001 to .1. See \cref{algo:aup} in \cref{sec:train} for more details.

\item[{\aupp}] {\aup}, but the auxiliary reward function is a random projection from a downsampled observation space to $\mathbb{R}$, without using a variational autoencoder. Since there is no \textsc{cb-vae} to learn, {\aupp} learns its Q-value network for the first 1M steps and trains on the $R_\textsc{aup}$ reward signal for the last 4M steps.

\item[{\naive}] Trained on the primary reward function $R$ minus (roughly) the $L_1$ distance between the current state and the initial state. The agent is penalized when cells differ from their initial values. We use an unscaled $L_1$ penalty, which \citet{wainwright2019safelife} found to produce the best results. 

While an $L_1$ penalty induces good behavior in certain static tasks, penalizing state change  often fails to avoid crucial side effects.  State change penalties do not differentiate between moving a box and irreversibly wedging a box in a corner \citep{relative_original}. 
\end{enumerate}




We only tuned hyperparameters on \texttt{append-still-easy}  before using them on all tasks. For \texttt{append-still}, we allotted an extra 1M steps to achieve convergence for all agents. For \texttt{append-spawn}, agents pretrain  on \texttt{append-still-easy} environments for the first 2M steps and train on \texttt{append-spawn} for the last 3M steps. For {\aup} in \texttt{append-spawn}, the autoencoder and auxiliary network are trained on both tasks. $R_\textsc{aup}$ is then pretrained for 2M steps and trained for 1.9M steps, thus training for the same total number of steps. 

\paragraph{Results.}

\begin{figure}[ht!]
\centering 
\caption{Smoothed learning curves with shaded regions representing $\pm 1$ standard deviation. {\aup} begins training on the $R_\textsc{aup}$ reward signal at step 1.1M, marked by a dotted vertical line. {\aupp} begins such training at step 1M.} 

\captionsetup[subfigure]{labelformat=empty, captionskip=-50pt}
\subfloat[][]{\label{reward-append-easy-plot}
    \includegraphics[height=1.95in]{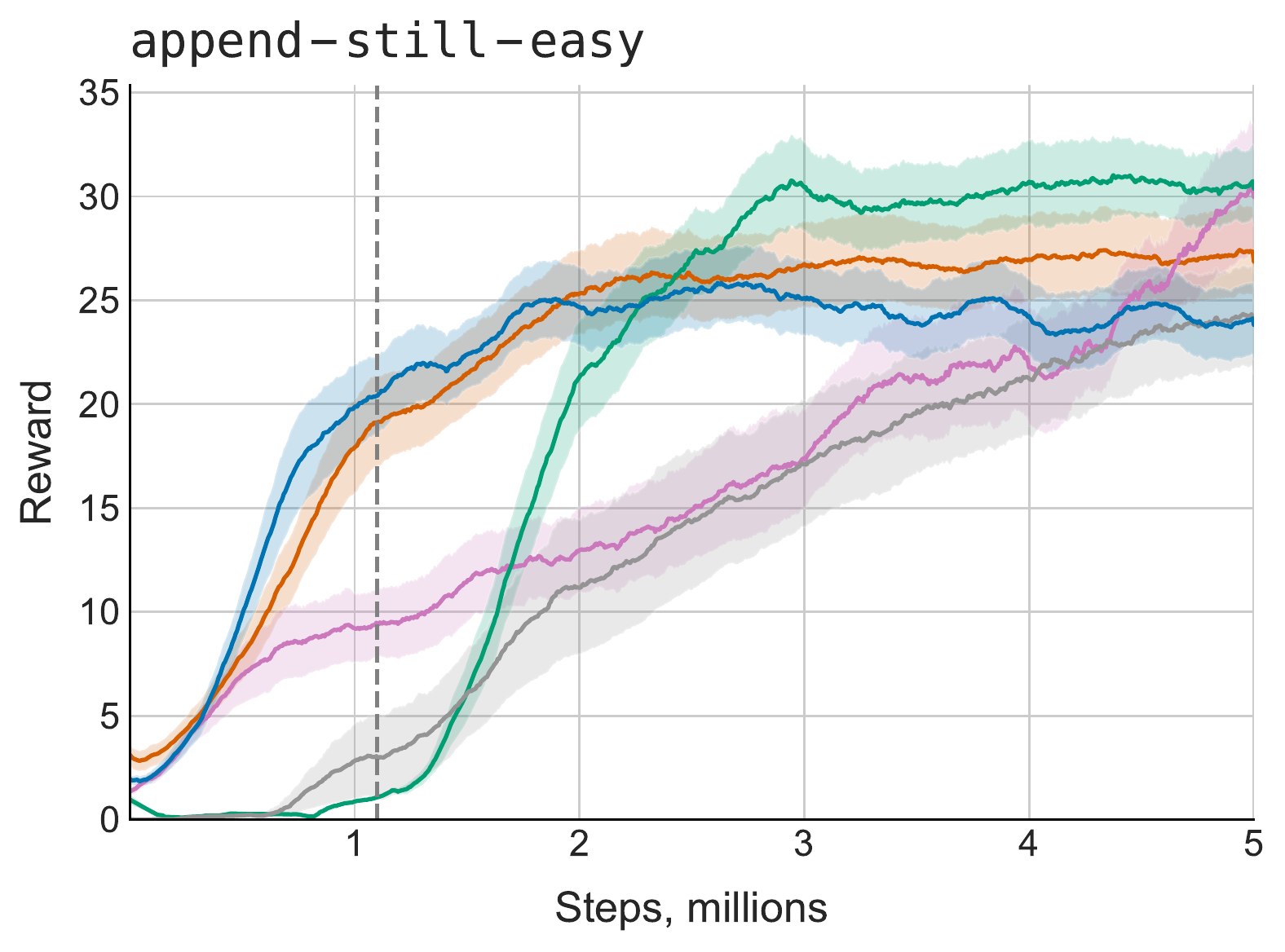}
}
\subfloat[][]{\label{side-append-easy-plot}
    \includegraphics[height=1.95in]{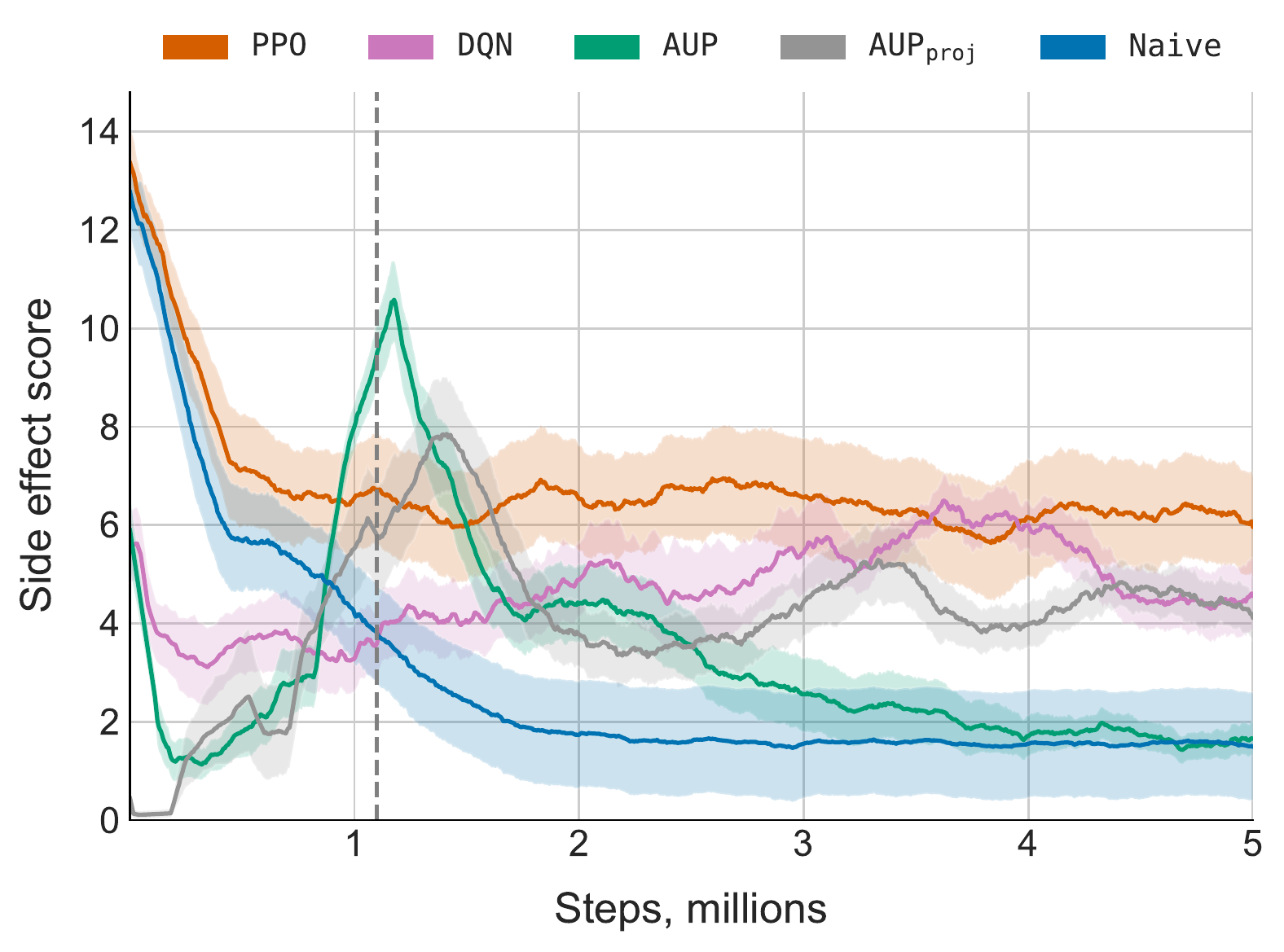}
}
\vspace{-8pt}
\subfloat[][]{\label{reward-prune-easy-plot}
    \includegraphics[height=1.95in]{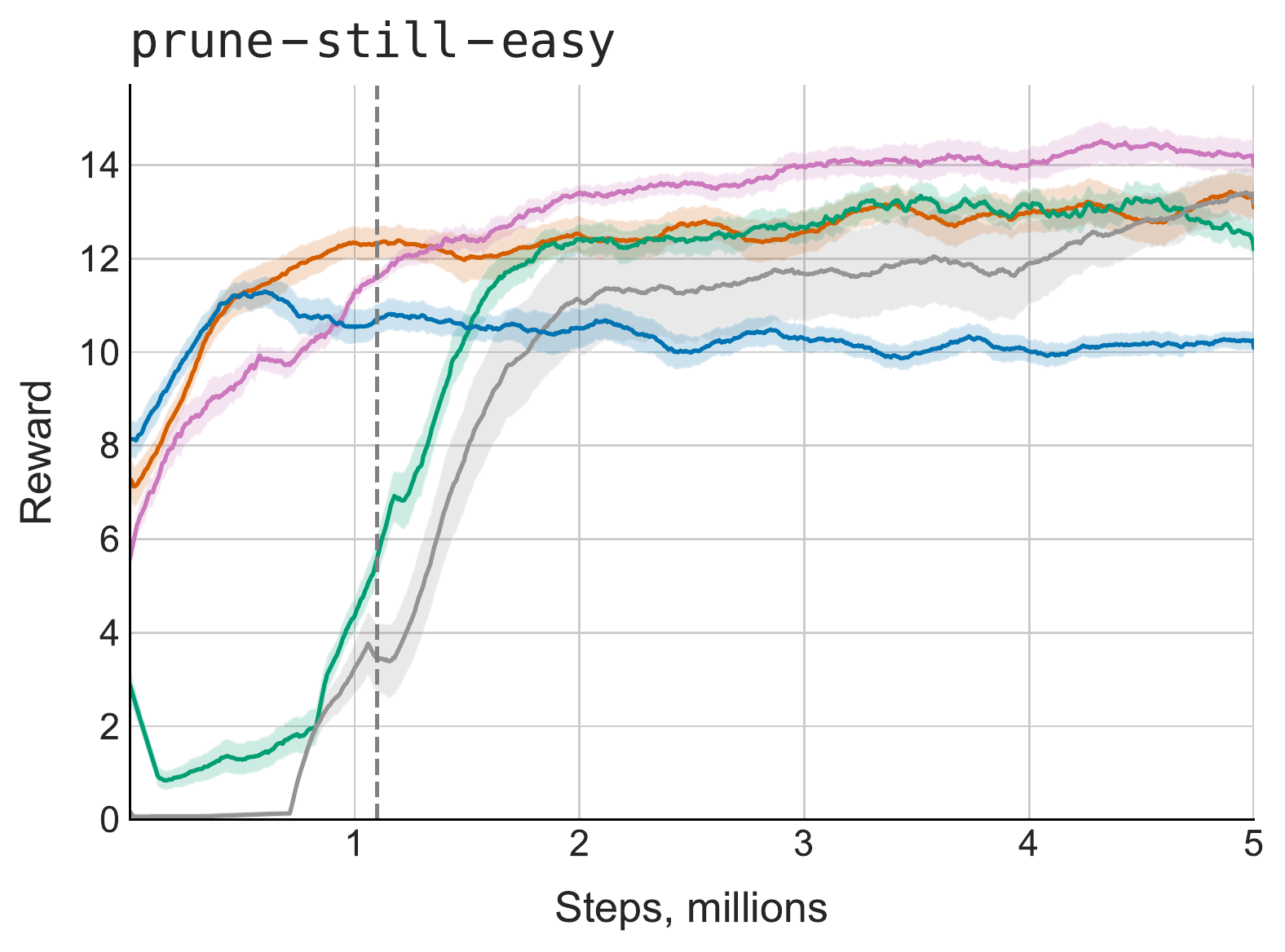}
}
\subfloat[][]{\label{side-prune-easy-plot}
    \includegraphics[height=1.95in]{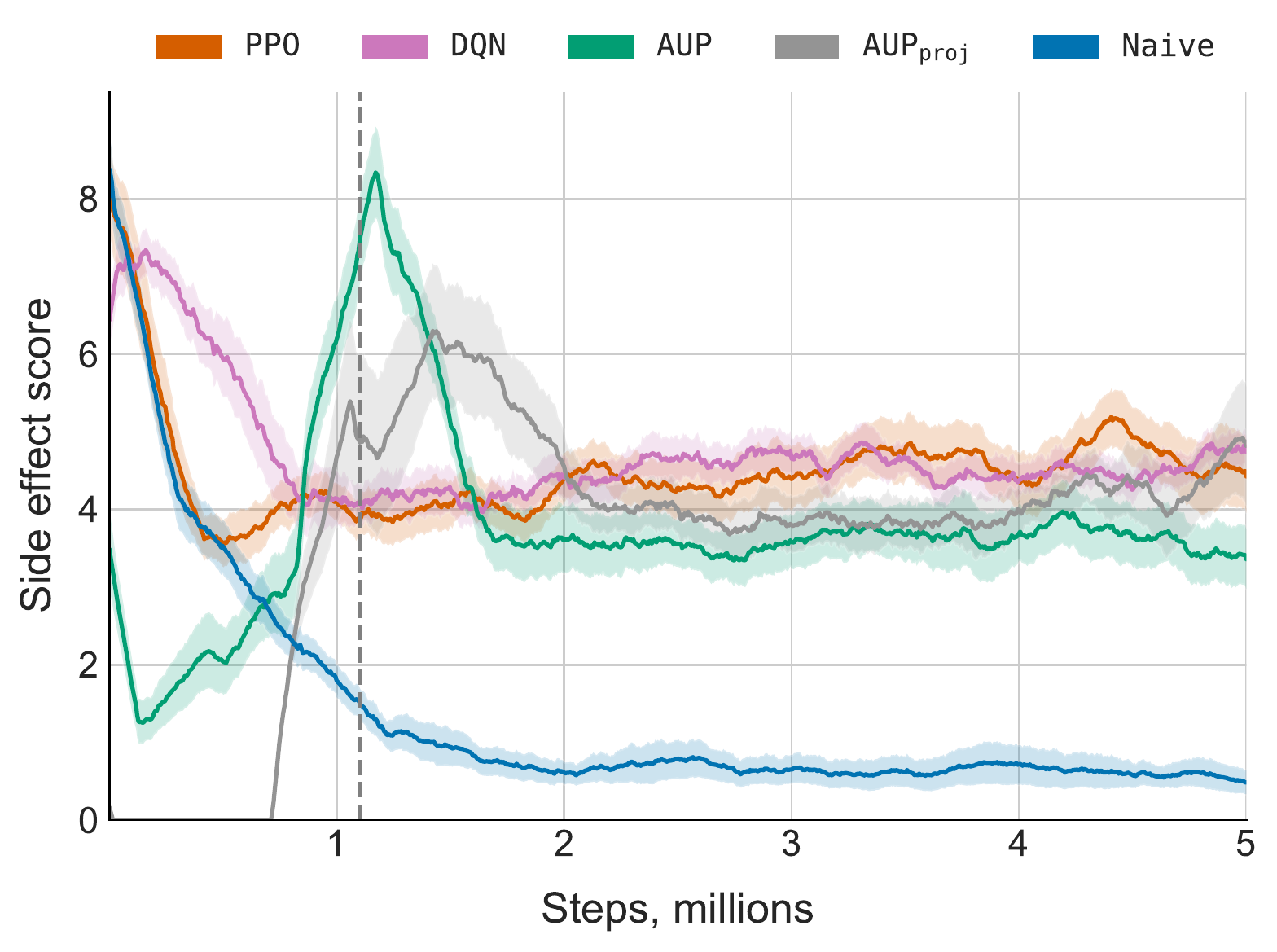}
}
\vspace{-8pt}

\subfloat[][]{\label{reward-append-plot}
    \includegraphics[height=1.95in]{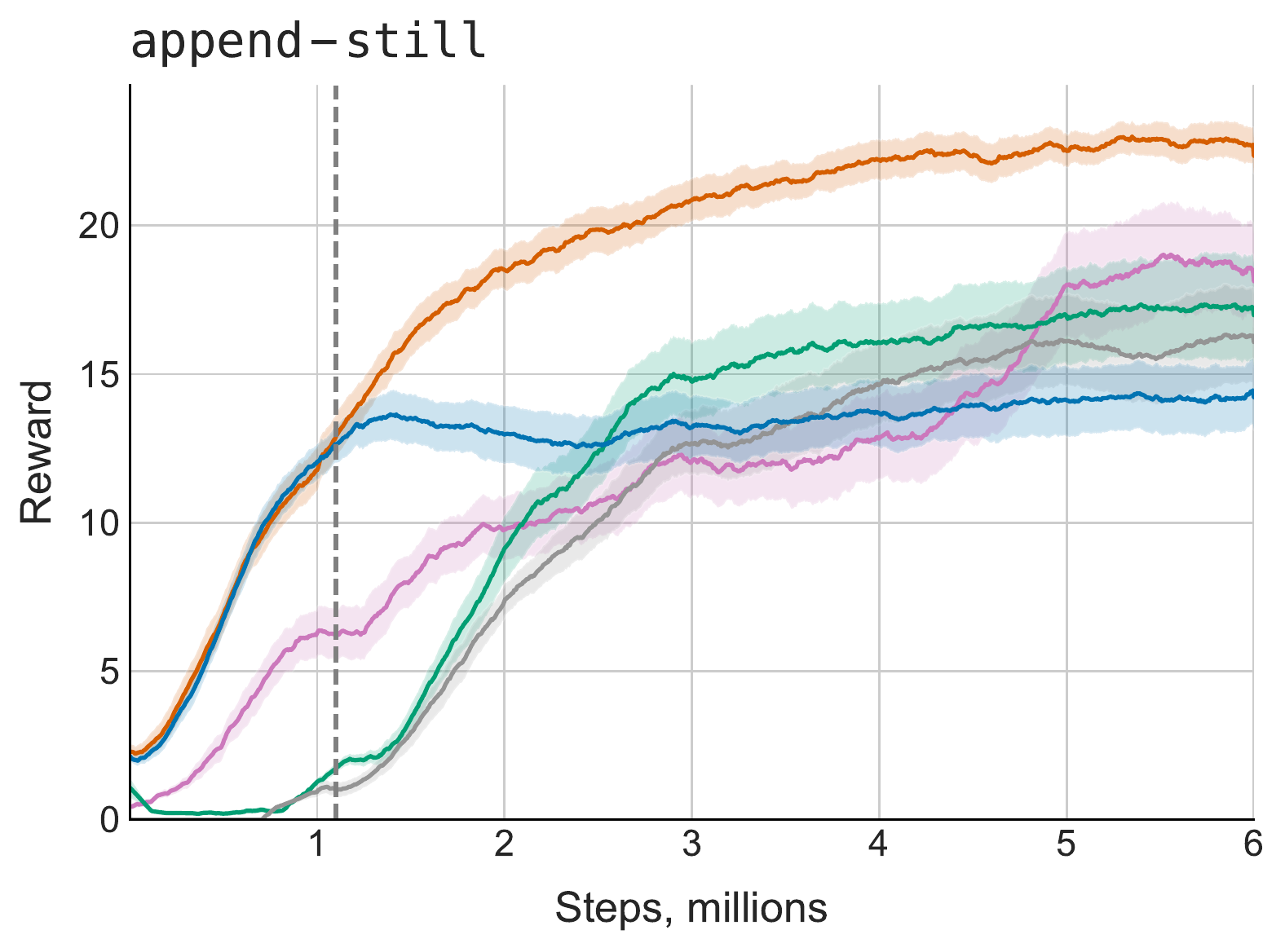}
}
\subfloat[][]{\label{side-append-plot}
    \includegraphics[height=1.95in]{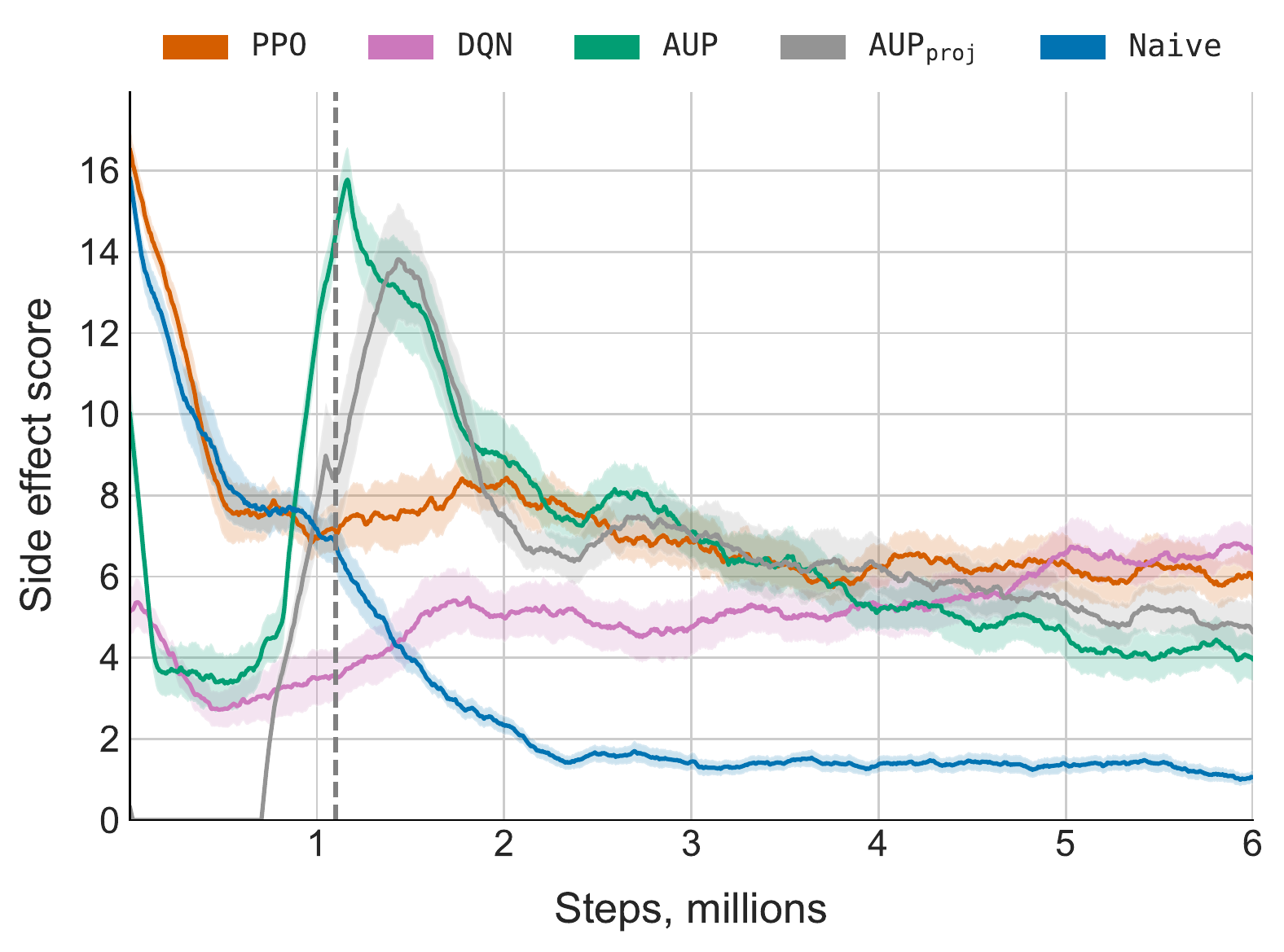}
}
\vspace{-8pt}
\subfloat[][]{\label{reward-append-spawn-plot}
    \includegraphics[height=1.95in]{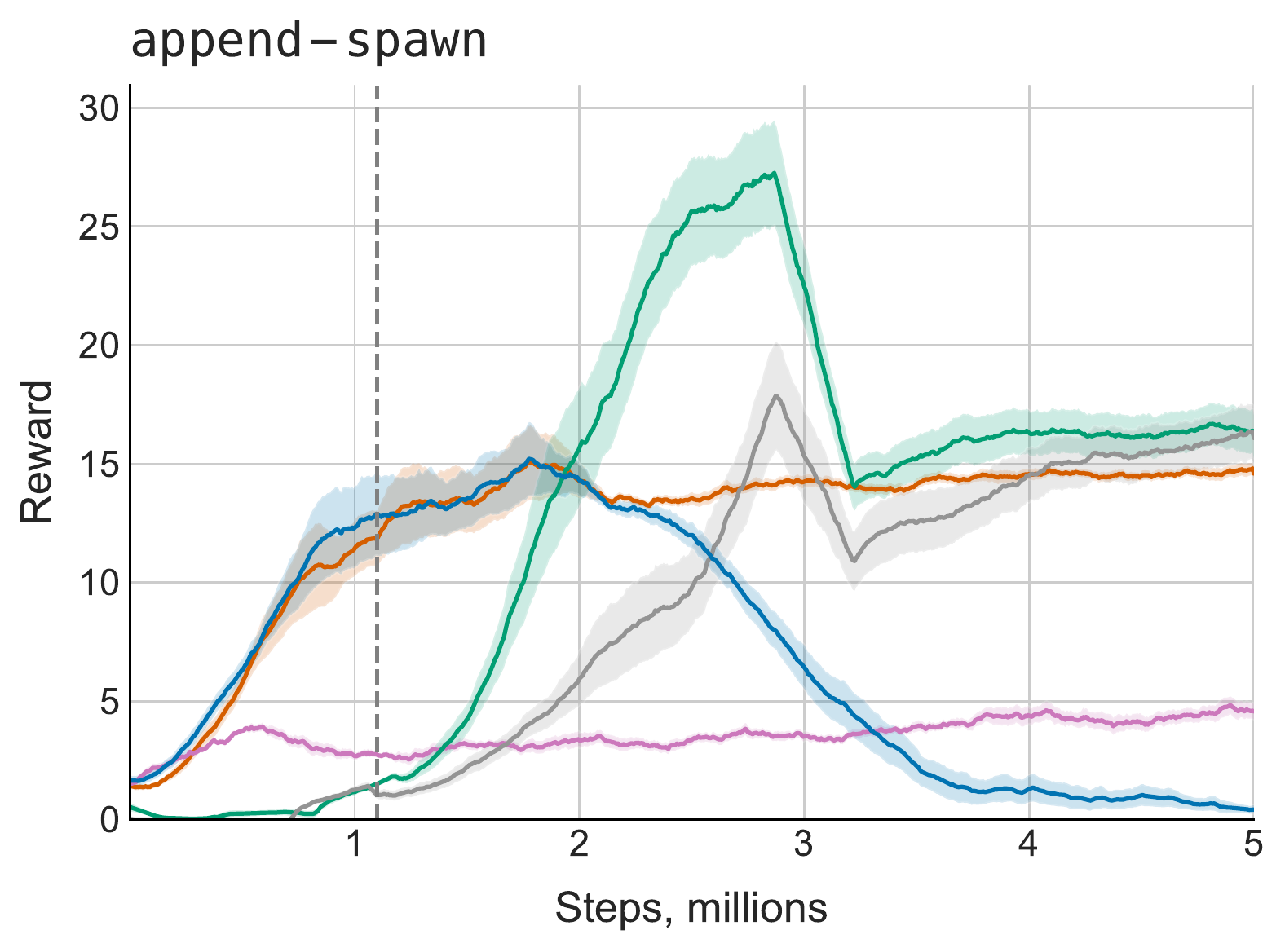}
}
\subfloat[][]{\label{side-append-spawn-plot}
    \includegraphics[height=1.95in]{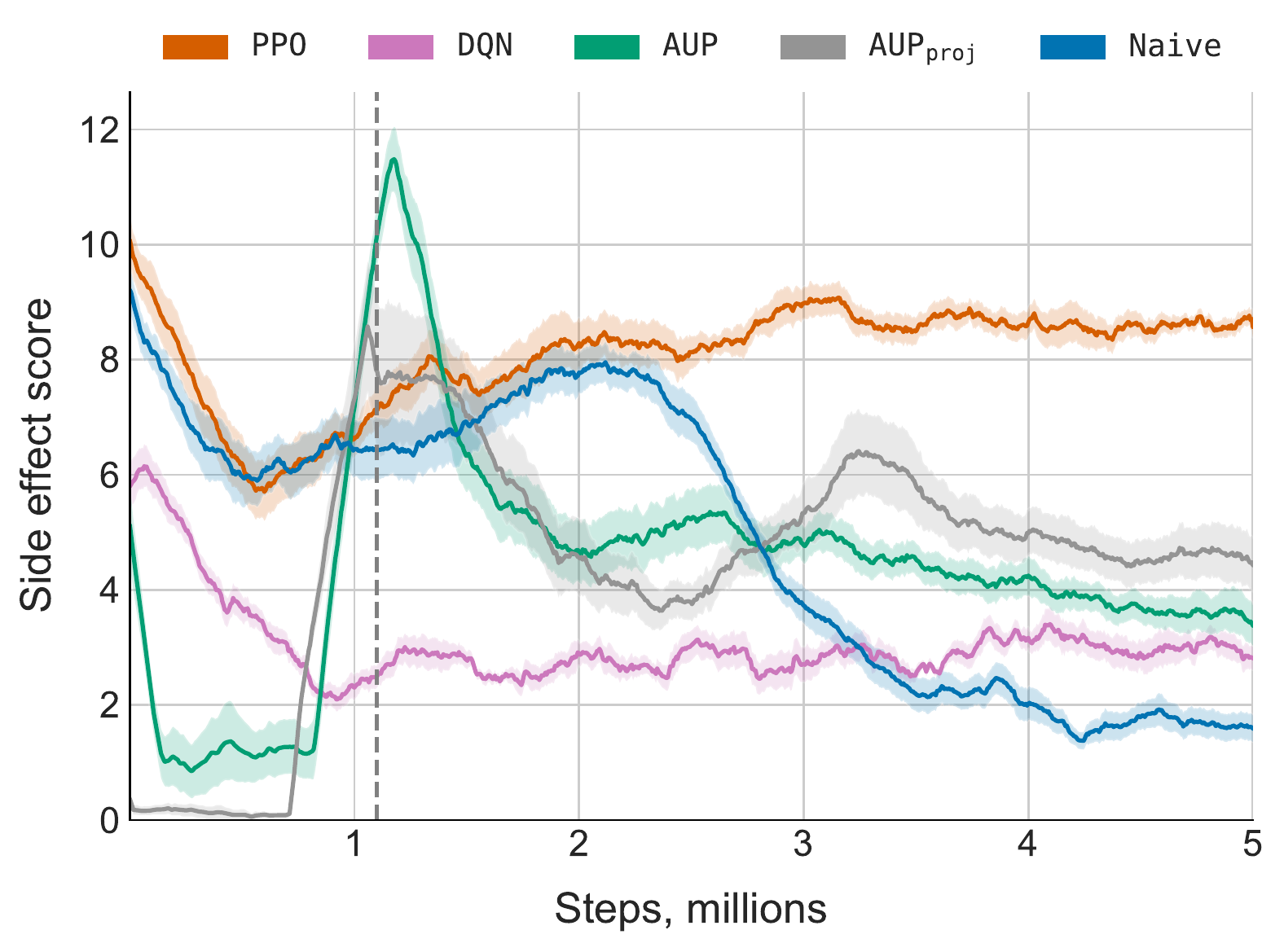}
}
\label{fig:plots}
\end{figure}

{\aup} learns quickly in \texttt{append-still-easy}. {\aup} waits 1.1M steps to start training on $R_\textsc{aup}$; while {\ppo} takes 2M steps to converge, {\aup} matches {\ppo} by step 2.5M and outperforms {\ppo} by step 2.8M (see \cref{fig:plots}). {\aup} and {\naive} both do well on side effects, with {\aup} incurring $27.8\%$ the side effects of {\ppo} after 5M steps. However, {\naive} underperforms {\aup} on reward.  {\dqn} learns more slowly, eventually exceeding {\aup} on reward. {\aupp} has lackluster performance, matching {\naive} on reward and {\dqn} on side effects, perhaps implying that the one-dimensional encoder  provides more structure than a random projection. 

In \texttt{prune-still-easy}, {\ppo}, {\dqn}, {\aup}, and {\aupp} all competitively accrue reward, while {\naive} lags behind. However, {\aup} only cuts out a quarter of {\ppo}'s side effects, while {\naive} does much better. Since all tasks but \texttt{append-spawn} are static, {\naive}'s $L_1$ penalty strongly correlates with the unobserved side effect metric (change to the green cells). {\aupp} brings little to the table, matching {\ppo} on both reward and side effects.

\texttt{append-still} environments contain more green cells than \texttt{append-still-easy} environments. By step 6M, {\aup} incurs $63\%$ of {\ppo}'s side effect score, while underperforming both {\ppo} and {\dqn} on reward. {\aupp} does slightly worse than {\aup} on both reward and side effects. Once again, {\naive} does worse than {\aup} on reward but better on side effects. In \cref{sec:addtl-data}, we display episode lengths over the course of training -- by step 6M, both {\aup} and {\naive} converge to an average episode length of about 780, while {\ppo} converges to 472.  

\texttt{append-spawn} environments contain stochastic yellow cell spawners. {\dqn} and {\aupp} both do extremely poorly. {\naive} usually fails to get \emph{any} reward, as it erratically wanders the environment. After 5M steps, {\aup} soundly improves on {\ppo}: 111\% of the reward, 39\% of the side effects, and 67\% of the  episode length. Concretely, {\aup} disturbs less than half as many green cells. Surprisingly, despite its middling performance on previous tasks, {\aupp} matches {\aup} on reward and cuts out 48\% of {\ppo}'s side effects.

\subsection{Hyperparameter sweep}
In the following, $N_\text{env}$ is the number of environments in the randomly generated curricula; when $N_\text{env}=\infty$, each episode takes place in a new environment. $Z$ is the dimensionality of the \textsc{cb-vae} latent space. While training on the $R_\textsc{aup}$ reward signal, the \textsc{aup} penalty coefficient $\lambda$ is linearly increased from .01 to $\lambda^*$. 

\paragraph{Method.}

In \texttt{append-still-easy}, we evaluate {\aup} on the following settings: $\lambda ^* \in \set{.1, .5, 1, 5}$, $\abs{\mathcal{R}}\in\set{1,2,4,8}$, and $(N_{\text{env}},Z) \in  \set{8,16,32,\infty}\times \set{1,4,16,64}$. We also evaluate  {\ppo} on each $N_{\text{env}}$ setting. We use default settings for all unmodified parameters.

For each setting, we record both the side effect score and the return of the learned policy, averaged over the last 100 episodes and over five seeds of each of three randomly generated  \texttt{append-still-easy} curricula. Curricula are held constant across settings with equal $N_\text{env}$ values.

After training the encoder, if $Z=1$, the auxiliary reward is the output of the encoder $E$. Otherwise, we draw linear functionals $\phi_i$ uniformly randomly from $(0,1)^{Z}$. The auxiliary reward function $R_i$ is defined as $\phi_i\circ E: \mathcal{S}\to \mathbb{R}$. 

For each of the $|\R|$ auxiliary reward functions, we learn a Q-value network for 1M time steps. The learned $Q_{R_i} $ define the penalty term of \cref{eq:aup}. While the agent trains on the $R_\textsc{aup}$ reward signal for the final 3.9M steps, $\lambda$ is linearly increased from .001 to $\lambda ^*$.

\begin{figure}[ht]
\centering 
\captionsetup[subfigure]{labelformat=empty}
\subfloat[][]{\label{s1}
    \includegraphics[height=1.7in,valign=t]{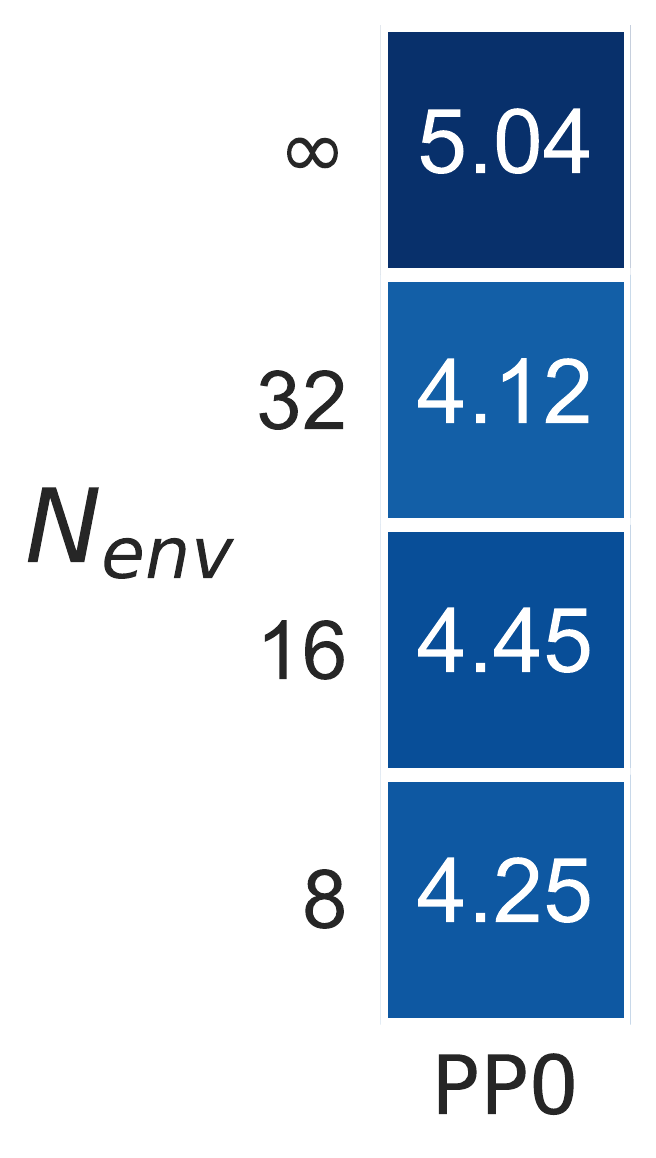}\hspace{-4pt}
    \includegraphics[height=1.9in,valign=t]{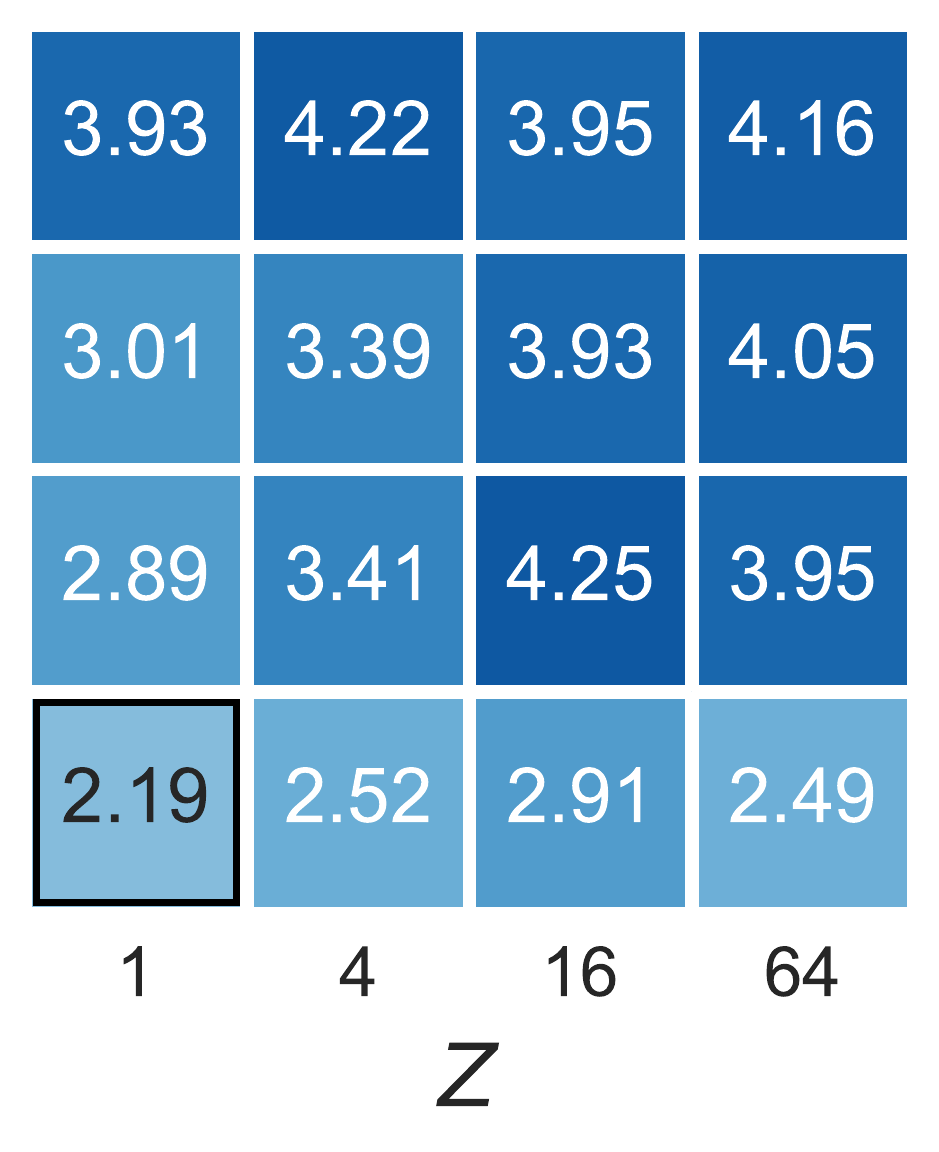}
} \quad \subfloat[][]{\label{s2}
    \includegraphics[height=1.9in, valign=t]{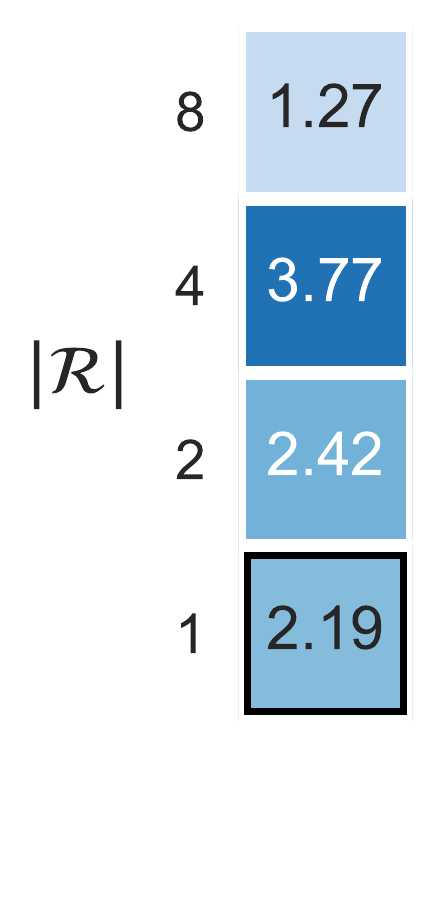}}  \qquad\subfloat[][]{\label{s3}
 \includegraphics[height=1.9in,valign=t]{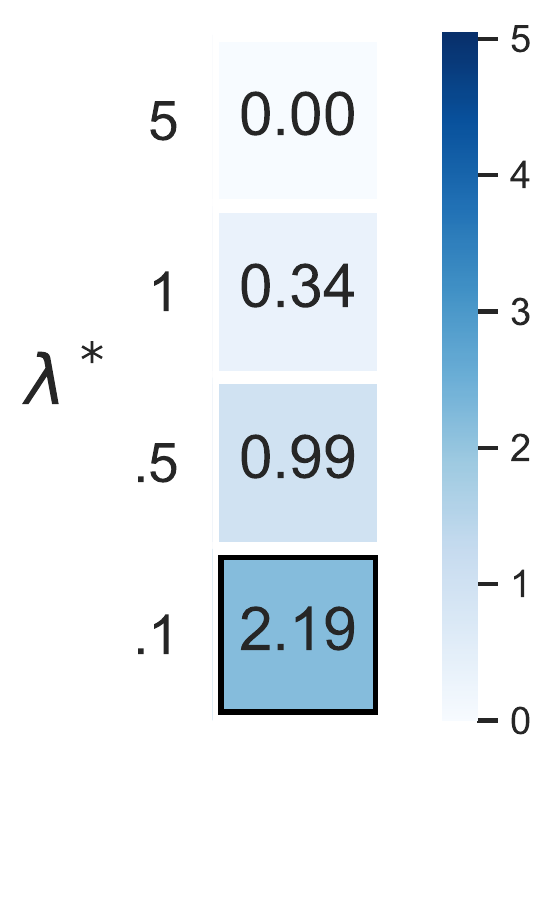}
}
\vspace{-10pt}
\caption{Side effect score for different {\aup} settings. Lower score is better. The default {\aup} setting ($Z=\abs{\R}=1,N_\text{env}=8,\lambda^*=.1$) is outlined in black. Unmodified hyperparameters take on their default settings; for example, when $\lambda^*=.5$ on the right, $Z=\abs{\R}=1,N_\text{env}=8$.\label{fig:side}}
\vspace{20pt}

\subfloat[][]{\label{r1}
    \includegraphics[height=1.7in,valign=t]{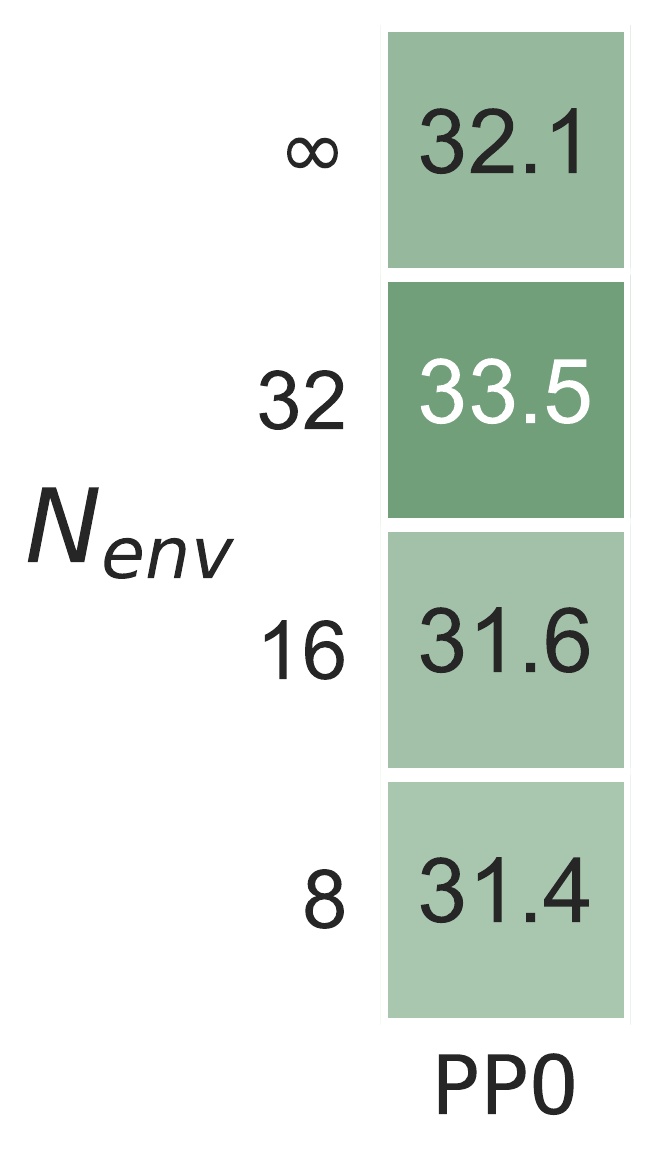}\hspace{-4pt}
    \includegraphics[height=1.9in,valign=t]{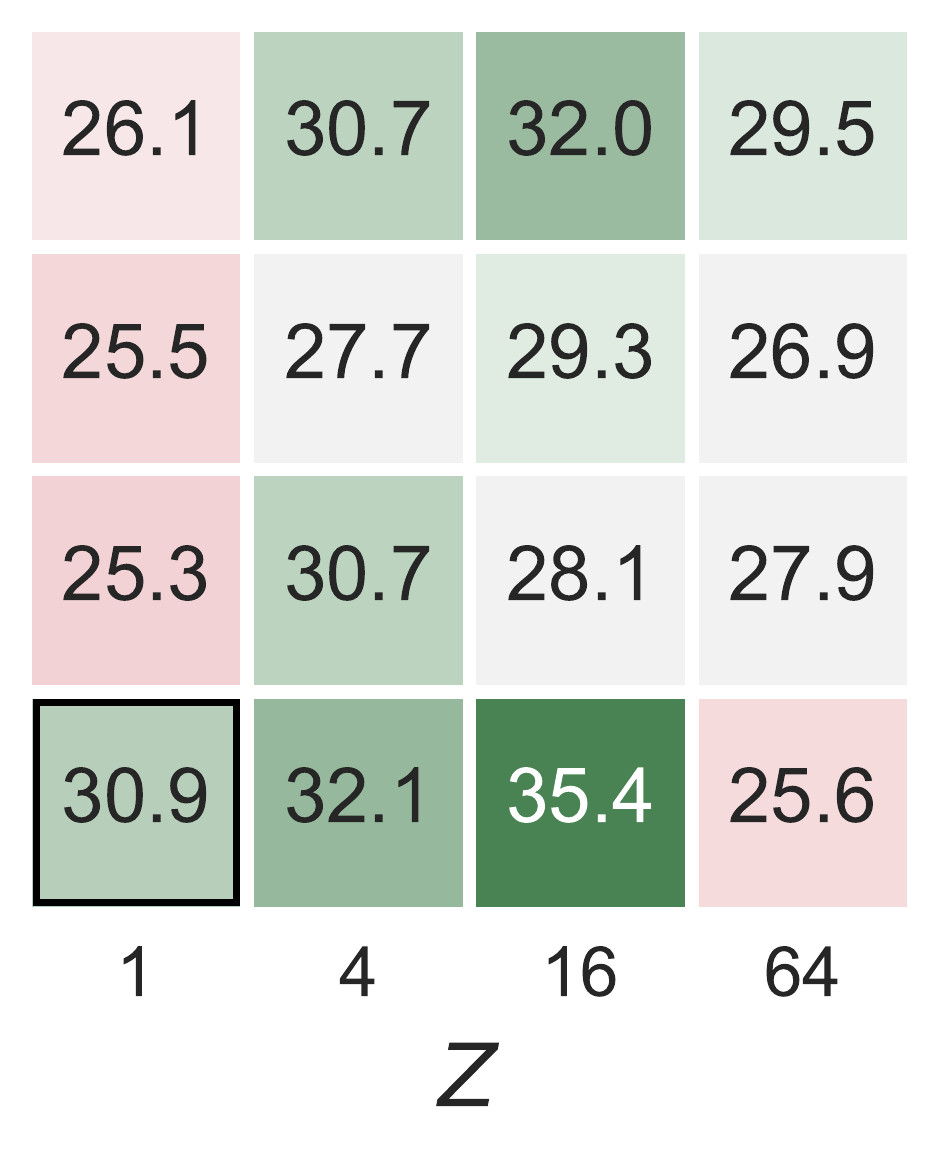}
} \quad \subfloat[][]{\label{r2}
    \includegraphics[height=1.9in, valign=t]{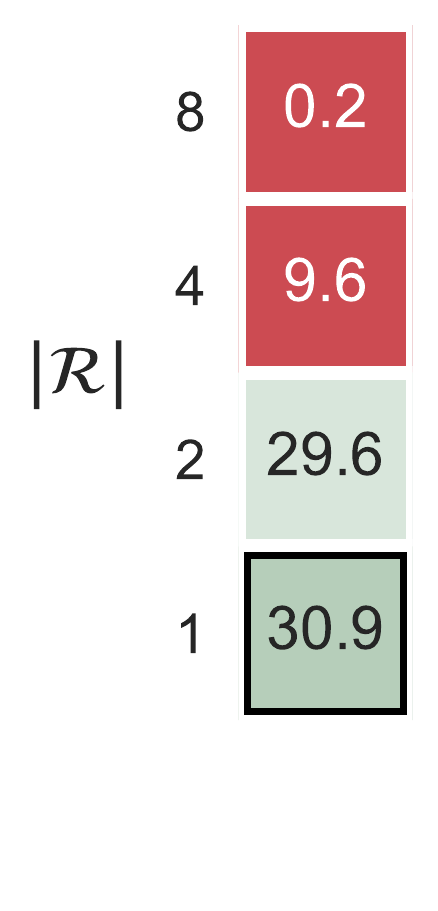}}  \qquad\subfloat[][]{\label{r3}
 \includegraphics[height=1.9in,valign=t]{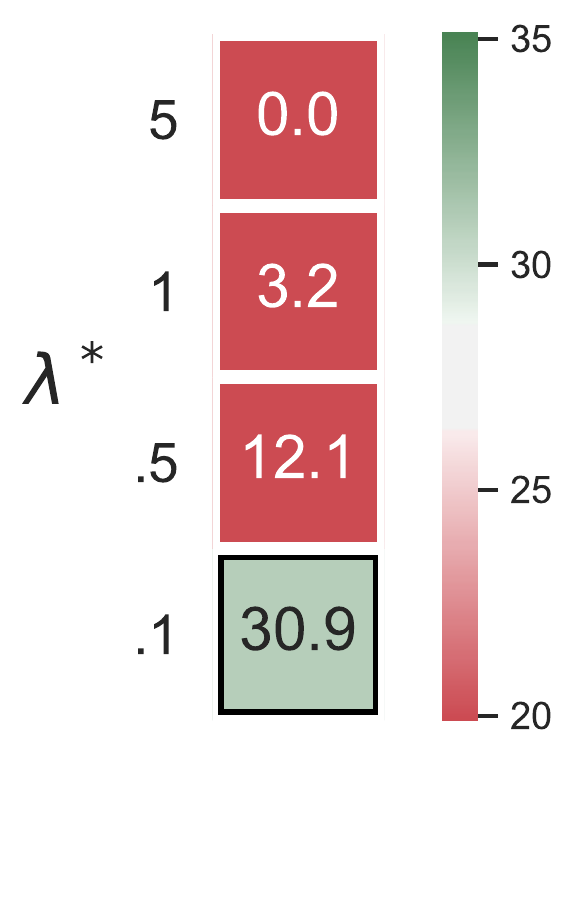}
}

\vspace{-10pt}
\caption{Episodic reward for different {\aup} settings. Higher reward  is better.\label{fig:reward}}
\end{figure}

\paragraph{Results.}

As $N_\text{env}$ increases, side effect score tends to increase. {\aup} robustly beats {\ppo} on side effect score: for each $N_\text{env}$ setting, {\aup}'s worst configuration has lower score than {\ppo}. Even when $N_\text{env}=\infty$, {\aup} ($Z=16$) shows the potential to significantly reduce side effects without reducing episodic return.  

{\aup} does well with a single latent space dimension ($Z=1$). For most $N_\text{env}$ settings, $Z$ is positively correlated with {\aup}'s side effect score. In \cref{sec:addtl-data}, data show that higher-dimensional auxiliary reward functions are harder to learn, presumably resulting in a poorly learned auxiliary Q-function. 

When $Z=1$, reward decreases as $N_\text{env}$ increases. We speculate that the \textsc{cb-vae} may be unable to properly encode large numbers of environments using only a 1-dimensional latent space. This would make the auxiliary reward function noisier and harder to learn, which could make the \textsc{aup} penalty term less meaningful. 

{\aup}'s default configuration achieves 98\% of {\ppo}'s episodic return, with just over half of the side effects. The fact that {\aup} is generally able to match {\ppo} in episodic reward leads us to hypothesize that the \textsc{aup} penalty term might be acting as a shaping reward. This would be intriguing – shaping usually requires knowledge of the desired task, whereas the auxiliary reward function is randomly generated. Additionally, after {\aup} begins training on the $R_\textsc{aup}$ reward signal at step 1.1M, {\aup} learns more quickly than {\ppo} did (\cref{fig:plots}), which supports the shaping hypothesis. {\aup} imposed minimal overhead: due to apparently increased sample efficiency, {\aup} reaches {\ppo}'s asymptotic episodic return at the same time as {\ppo} in \texttt{append-still-easy} and \texttt{append-spawn} (\cref{fig:plots}).

Surprisingly, {\aup} does well with a single auxiliary reward function ($\abs{\R}=1$). We hypothesize that destroying patterns decreases optimal value for a wide range of reward functions. Furthermore, we suspect that decreasing optimal value in general often decreases optimal value for any given single auxiliary reward function. In other words, we suspect that optimal value at a state is heavily correlated across reward functions, which might explain \citet{schaul_universal_2015}'s success in learning regularities across value functions. This potential correlation might explain why {\aup} does well with one auxiliary reward function. 

We were surprised by the results for $\abs{\R}=4$ and $\abs{\R}=8$. In \citet{turner2020conservative}, increasing $\abs{\R}$ reduced the number of side effects without impacting performance on the primary objective. We believe that work on better interpretability of AUP's $Q_{R_i}$ will increase understanding of these results.

When $\lambda^*=.5$, {\aup} becomes more conservative. As $\lambda^*$ increases further, the learned \texttt{AUP} policy stops moving entirely. 

\section{Discussion}

We successfully scaled \textsc{aup} to complex environments without providing task-specific knowledge -- the auxiliary reward function was a one-dimensional variational autoencoder trained through uniformly random exploration. To the best of our knowledge, \textsc{aup} is the first task-agnostic approach which reduces side effects and competitively achieves reward in complex environments. 

\citet{wainwright2019safelife} speculated that avoiding side effects must necessarily decrease performance on the primary task. This may be true for optimal policies, but not necessarily for learned policies. {\aup} improved performance on \texttt{append-still-easy} and \texttt{append-spawn}, matched performance on \texttt{prune-still-easy}, and underperformed on \texttt{append-still}. Note that since \textsc{aup} only regularizes learned policies, {\aup}  can still make expensive mistakes during training. 

{\aupp} worked well on \texttt{append-spawn}, while only slightly reducing side effects on the other tasks. This suggests that \textsc{aup} works (to varying extents) for a wide range of uninformative reward functions. 

While {\naive} penalizes every state perturbation equally, \textsc{aup} theoretically applies penalty in proportion to irreversibility (\cref{mut-bound}). For example, the agent could move crates around (and then put them back later). {\aup} incurred little penalty for doing so, while {\naive} was more constrained and consistently earned less reward than {\aup}. We believe that \textsc{aup} will continue to scale to  useful applications, in part because it naturally accounts for irreversibility.

\paragraph{Future work.}

Off-policy learning could allow simultaneous training of the auxiliary $R_i$ and of $R_\textsc{aup}$. Instead of learning an auxiliary Q-function, the agent could just learn the auxiliary advantage function with respect to inaction. 

Some environments do not have a no-op action, or the agent may have more spatially distant effects on the world which are not reflected in its auxiliary action values. In addition, separately training the auxiliary networks may be costly, which might necessitate off-policy learning.  We look forward to future work investigating these challenges.

The SafeLife suite includes more challenging variants of \texttt{prune-still-easy}. SafeLife also includes difficult \texttt{navigation} tasks, in which the agent must reach the goal by wading either through fragile green patterns or through robust yellow patterns. Additionally, \textsc{aup} has not yet been evaluated in partially observable domains.

{\aup}'s strong performance when $|\R|=Z=1$ raises interesting questions. \citet{turner2020conservative}'s small ``\texttt{Options}'' environment required $\abs{\R}\approx 5$ for good performance. SafeLife environments are much larger than \texttt{Options} (\cref{tab:comparison}), so why does $\abs{\R}=1$ perform well, and why does $\abs{\R}>2$ perform poorly? To what extent does the \textsc{aup} penalty term provide reward shaping?  Why do one-dimensional encodings provide a learnable reward signal over states?

\paragraph{Conclusion.}
To realize the full potential of \textsc{rl}, we need more than algorithms which train policies -- we need to be able to train policies which actually do what we want. Fundamentally, we face a frame problem: we often know what we want the agent to do, but we cannot list everything we want the agent \emph{not} to do. \textsc{Aup} scales to challenging domains, incurs modest overhead, and induces competitive performance on the original task while significantly reducing side effects -- without explicit information about what side effects to avoid. 

\section*{Broader Impact}

A scalable side effect avoidance method would ease the challenge of reward specification and aid deployment of \textsc{rl} in situations where mistakes are costly, such as embodied robotics tasks for which sim2real techniques are available.  Conversely, developers should  carefully consider how \textsc{rl} algorithms might produce policies with catastrophic impact. Developers should not blindly rely on  even a well-tested side effect penalty. 

In some applications, \textsc{aup} may decrease performance on the primary objective. In this case, developers may be incentivized to ``cut corners'' on safety in order to secure competitive advantage. 

\begin{ack}
We thank Joshua Turner for help compiling \cref{fig:side}  and \cref{fig:reward}. We thank Chase Denecke for improving \href{https://avoiding-side-effects.github.io/}{https://avoiding-side-effects.github.io/}. We thank Stuart Armstrong, Andrew Critch, Chase Denecke, Evan Hubinger, Victoria Krakovna, Dylan Hadfield-Menell, Matthew Olson, Rohin Shah, Logan Smith, and Carroll Wainwright for their ideas and feedback. We thank our reviewers for improving the paper with their detailed comments.

Alexander Turner is supported by a grant from the Long-Term Future Fund via the Center for Effective Altruism, and reports additional revenues related to this work from employment at the Center for Human-Compatible AI.
Neale Ratzlaff was partially supported by the Defense Advanced Research Projects Agency (DARPA) under Contract No. HR001120C0011 and HR001120C0022. Prasad Tadepalli acknowledges the support of NSF grants IIS-1619433, IIS-1724360, and DARPA contract N66001-19-2-4035. Any opinions, findings and conclusions or recommendations expressed in this material are those of the author(s) and do not necessarily reflect the views of DARPA. 
\end{ack}

\bibliographystyle{plainnat}
\bibliography{aup_rl.bib}
\newpage
\appendix

\section{Theoretical Results}
\label{sec:theory}

Consider a rewardless \textsc{mdp} $\langle \mathcal{S},\mathcal{A},T,\gamma\rangle$. 
Reward functions $R\in\mathbb{R}^\St$ have corresponding optimal value functions $\Vf{s}$.

\begin{prop}[Communicability bounds maximum change in optimal value]\label{mut-bound}
If $s$ can reach $s'$ with probability 1 in $k_1$ steps and $s'$ can reach $s$ with probability 1 in $k_2$ steps, then $\max_{R\in [0, 1]^{\St}} \left|V^*_R(s)-V^*_R(s')\right|\leq \geom[1-\gamma^{\max(k_1,k_2)}]< \geom$.
\end{prop}
\begin{proof}
We first bound the maximum increase.
\begin{align}
    \max_{R\in [0, 1]^{\St}} V^*_R(s')-V^*_R(s) &\leq \max_{R\in [0, 1]^{\St}} V^*_R(s') - \left(0\cdot \geom[1-\gamma^{k_1}]+\gamma^{k_1}V^*_R(s')\right)\label{eq:max-inc}\\
    &\leq \geom[1] - \left(0\cdot \geom[1-\gamma^{k_1}]+\gamma^{k_1}\geom[1]\right)\\
    &=\geom[1-\gamma^{k_1}].
\end{align}
\Cref{eq:max-inc} holds because even if we make $R$ equal  $0$ for as many states as possible, $s'$ is still reachable from $s$. The case for maximum decrease is similar.
\end{proof}

\section{Training Details}\label{sec:train}
In \cref{sec:line}, we aggregated performance from 3 curricula with 5 seeds each, and 1 curriculum with 3 seeds.

We detail how we trained the {\aup} and {\aup}$_\text{proj}$ conditions. An algorithm describing the training process can be seen in algorithm \ref{algo:aup}.

\subsection{Auxiliary reward training}

For the first phase of training, our goal is to learn $Q_{\text{aux}}$, allowing us to compute the \textsc{aup} penalty in the second phase of training. Due to the size of the full SafeLife state $(350 \times 350 \times 3)$, both conditions downsample the observations with average pooling and convert to intensity values. 

Previously, \citet{turner2020conservative} learned $Q_{\text{aux}}$ with tabular Q-learning. They used environments small enough such that reward could be assigned to each state. Because SafeLife environments are too large for tabular Q-learning, we demonstrated two methods for randomly generating an auxiliary reward function.

\begin{enumerate}
\item[{\aup}] We acquire a low-dimensional state representation by training a continuous Bernoulli variational autoencoder  \citep{loaiza2019continuous}. To train the \textsc{cb-vae}, we collect a buffer of observations by acting randomly for $\frac{100,000}{N_{\text{env}}}$ steps in each of the $N_{\text{env}}$ environments. This gives us 100K total observations with an $N_{\text{env}}$-environment curriculum.  We train the \textsc{cb-vae} for 100 epochs, preserving the encoder $E$ for downstream auxiliary reward training.  

For each auxiliary reward function, we draw a linear functional uniformly from $(0, 1)^Z$ to serve as our auxiliary reward function, where $Z$ is the dimension of the \textsc{cb-vae}'s latent space. The auxiliary reward for an observation is the composition of the linear functional with an observation's latent representation. 

\item[{\aupp}] Instead of using a \textsc{cb-vae}, {\aup}$_\text{proj}$ simply downsamples the input observation. At the beginning of training, we generate a linear functional over the unit hypercube (with respect to the downsampled observation space). The auxiliary reward for an observation is the composition of the linear functional with the downsampled observation.
\end{enumerate}

The auxiliary reward function is learned after it is generated. To learn $Q_{\text{aux}}$, we modify the value function in \textsc{ppo} to a Q-function. Our training algorithm for phase 1 only differs from \textsc{ppo} in how we calculate reward. We train each auxiliary reward function for 1M steps.

\subsection{\textsc{Aup} reward training}
In phase 2, we train a new \textsc{ppo} agent on $R_\textsc{aup}$ (\cref{eq:aup}) for the corresponding SafeLife task. Each step, the agent selects an action $a$ in state $s$ according to its policy $\pi_{{\aup}}$, and receives reward $R_\textsc{aup}(s,a)$ from the environment. We compute $R_\textsc{aup}(s,a)$ with respect to the learned Q-values $Q_{\text{aux}}(s, \varnothing)$ and $Q_{\text{aux}}(s, a)$.

The penalty term is modulated by the hyperparameter $\lambda$, which is linearly scaled from $10^{-3}$ to some final value $\lambda^*$ (default $10^{-1}$). Because $\lambda$ controls the relative influence of the penalty, linearly increasing $\lambda$ over time will prioritize primary task learning in early training and slowly encourage the agent to obtain the same reward while avoiding side effects. If $\lambda$ is too large -- if side effects are too costly -- the agent won't have time to adapt its current policy and will choose inaction ($\varnothing$) to escape the penalty. A careful $\lambda$ schedule helps induce a successful policy that avoids side effects.

\begin{algorithm}[H]
    \SetAlgoLined
    \textbf{Initialize} Exploration buffer $S$\\
    \textbf{Initialize} CB-VAE $\mathcal{F}$ with encoder $E$, decoder $D$ \\    \textbf{Initialize} Exploration buffer $S$\\
    \textbf{Initialize} Auxiliary reward functions $\phi$ \\
    \textbf{Initialize} Auxiliary policy $\psi_{\text{aux}}$, AUP policy $\pi_{\text{AUP}}$\\
    \textbf{Require} CB-VAE training epochs $T$\\
    \textbf{Require} AUP penalty $\lambda$ \\
    \textbf{Require} Exploration buffer size $k$\\
    \textbf{Require} Auxiliary model training steps $L$\\
    \textbf{Require} AUP model training steps $N$\\
    \textbf{Require} PPO update function PPO-Update\\
    \textbf{Require} CB-VAE update function VAE-Update\\

    \For{Step $k = 1, \dots K$}{
        Sample random action $a$ \\
        $s \leftarrow \text{Act}(a)$\\
        $S = s \cup S$ \\
    }
    \For{Epoch $t = 1, \dots T$}{
        \normalfont{Update-VAE $(\mathcal{F}, S)$} \\
    }
    \For{Step $i = 1, \dots L+N$}{
        $s \leftarrow$ Starting state\\
        \For{Step $l = 1, \dots L$}{
            $a = \psi_{\text{aux}}(s)$ \\
            $s' = \text{ Act}(a)$ \\
            $r = \phi \cdot E(s)$ \\
            PPO-Update $(\psi_{\text{aux}}, s, a, r, s')$\\
            $s = s'$ \\
        }
        $s \leftarrow$ Starting state\\
        \For{Step $n = 1, \dots N$}{
            $a = \pi_{\text{AUP}}(s)$ \\
            $s', r = \text{ Act}(a)$ \\
            $r = r + R_{\text{AUP}}(\psi_{\text{aux}}, \, \pi_{\text{AUP}}, \,s, a, \lambda)$ \quad (\Cref{eq:aup}) \\
            PPO-Update $(\pi_{\text{AUP}}, s, a, r, s')$\\
            $s = s'$ \\
            
        }
     }
    \caption{\textsc{Aup} Training Algorithm}
    \label{algo:aup}
\end{algorithm}
\vspace{-0.05in}

\section{Hyperparameter Selection}
Table \ref{tab:params} lists the hyperparameters used for all conditions, which generally match the default SafeLife settings. \emph{Common} refers to those hyperparameters that are the same for each evaluated condition. \emph{AUX} refers to hyperparameters that are used only when training on $R_\text{AUX}$, thus, it only pertains to {\aup} and {\aup}$_\text{proj}$. The conditions {\ppo} and {\naive} use the \emph{PPO} hyperparameters for the duration of their training, while {\aup}, {\aup}$_\text{proj}$ use them when training with respect to $R_\textsc{aup}$. \emph{DQN} refers to the hyperparameters used to train the model for {\dqn}. 

\begin{longtable}[ht]{rr l }
        \toprule
        \multicolumn{2}{r}{Hyperparameter} &  Value\\
        \midrule
        \multicolumn{2}{r}{\it{Common}}& \\
        & Learning Rate & $3 \cdot 10^{-4}$\\
        & Optimizer & Adam \\
        & Gamma ($\gamma$) &  $0.97$\\
        & Lambda (\textsc{ppo}) & $0.95$\\
        & Lambda (\textsc{aup}) & $10^{-3} \rightarrow 10^{-1}$\\
        & Entropy Clip & $1.0$\\
        & Value Coefficient & $0.5$\\
        & Gradient Norm Clip & $5.0$\\
        & Clip Epsilon & $0.2$ \\
        
        \midrule
        \multicolumn{2}{r}{\emph{AUX}}& \\
        & Entropy Coefficient & $0.01$\\
        & Training Steps & $1 \cdot 10^6$\\

        \midrule
        \multicolumn{2}{r}{\emph{AUP\textsubscript{proj}}}& \\
        & Lambda (\textsc{aup}) & $10^{-3}$\\

        \midrule
        \multicolumn{2}{r}{\emph{PPO}}& \\
        & Entropy Coefficient & $0.1$\\
        
        \midrule
        \multicolumn{2}{r}{\emph{DQN}}& \\
        & Minibatch Size & $64$ \\
        & SGD Update Frequency & $16$\\
        & Target Network Update Frequency & $1 \cdot 10^3$\\
        & Replay Buffer Capacity & $1 \cdot 10^4$\\
        & Exploration Steps & $4 \cdot 10^3$\\
        
        \midrule
        \multicolumn{2}{r}{\emph{Policy}}& \\
        & Number of Hidden Layers & $3$\\
        & Output Channels in Hidden Layers & $(32, 64, 64)$\\
        & Nonlinearity & ReLU\\
        \midrule
        \multicolumn{2}{r}{\emph{cb-vae}}& \\
        & Learning Rate & $10^{-4}$\\
        & Optimizer & Adam \\
        & Latent Space Dimension ($Z$) & $1$\\
        & Batch Size & $64$ \\ 
        & Training Epochs & $50$\\
        & Epsilon & $10^{-5}$\\
        
        & Number of Hidden Layers (encoder) & $6$\\
        & Number of Hidden Layers (decoder) & $5$\\
        & Hidden Layer Width (encoder) & $(512, 512, 256, 128, 128, 128)$\\
        & Hidden Layer Width (decoder) & $(128, 256, 512, 512, \text{output})$\\
        & Nonlinearity & \textsc{elu}\\
        \bottomrule
   \caption{Chosen hyperparameters.\label{tab:params}}
\end{longtable}
\section{Compute Environment}
\begin{wraptable}{r}{6cm}
    \centering
    \vspace{-40pt}
    \begin{tabular}{cc}
    \toprule
    Condition & GPU-hours per trial\\
    \midrule
    {\ppo} & 6 \\
    
    {\dqn} & 16 \\
    
    {\aup} & 8 \\
    
    {\aupp} & 7.5 \\
    
    {\naive} & 6\\
    \bottomrule
    \end{tabular}
    \caption{Compute time for each condition.}
    \vspace{-25pt}
    \label{tab:hours}
\end{wraptable}

For data collection, we only ran the experiments once. All experiments were performed on a combination of NVIDIA GTX 2080TI GPUs, as well as NVIDIA V100 GPUs. No individual experiment required more than 3GB of GPU memory. We did not run a 3-seed {\dqn} curriculum for the experiments in \cref{sec:line}.

The auxiliary reward functions were trained on down-sampled rendered game screens, while all other learning used the internal SafeLife state representation. Incidentally, \cref{tab:hours} shows that \texttt{AUP}'s preprocessing turned out to be computationally expensive (compared to \texttt{PPO}'s).

\section{Additional Data} \label{sec:addtl-data}

\Cref{fig:length} plots episode length and \cref{fig:aux-reward} plots auxiliary reward learning. \Cref{fig:individual} and \cref{fig:individ-length} respectively plot reward/side effects and episode lengths for each {\aup} seed. \Cref{fig:batch} and \cref{fig:batch-length} plot the same, averaged over each curriculum; these data suggest that {\aup}'s performance is sensitive to the randomly generated curriculum of environments.

\begin{figure}[h]
\centering 
\caption{Smoothed episode length curves with shaded regions representing $\pm 1$ standard deviation.  {\aup} and {\aupp} begin training on the $R_\textsc{aup}$ reward signal at steps 1.1M and 1M,  respectively.}\label{fig:length}

\captionsetup[subfigure]{labelformat=empty, captionskip=-50pt}
\subfloat[][]{
    \includegraphics[height=2.1in]{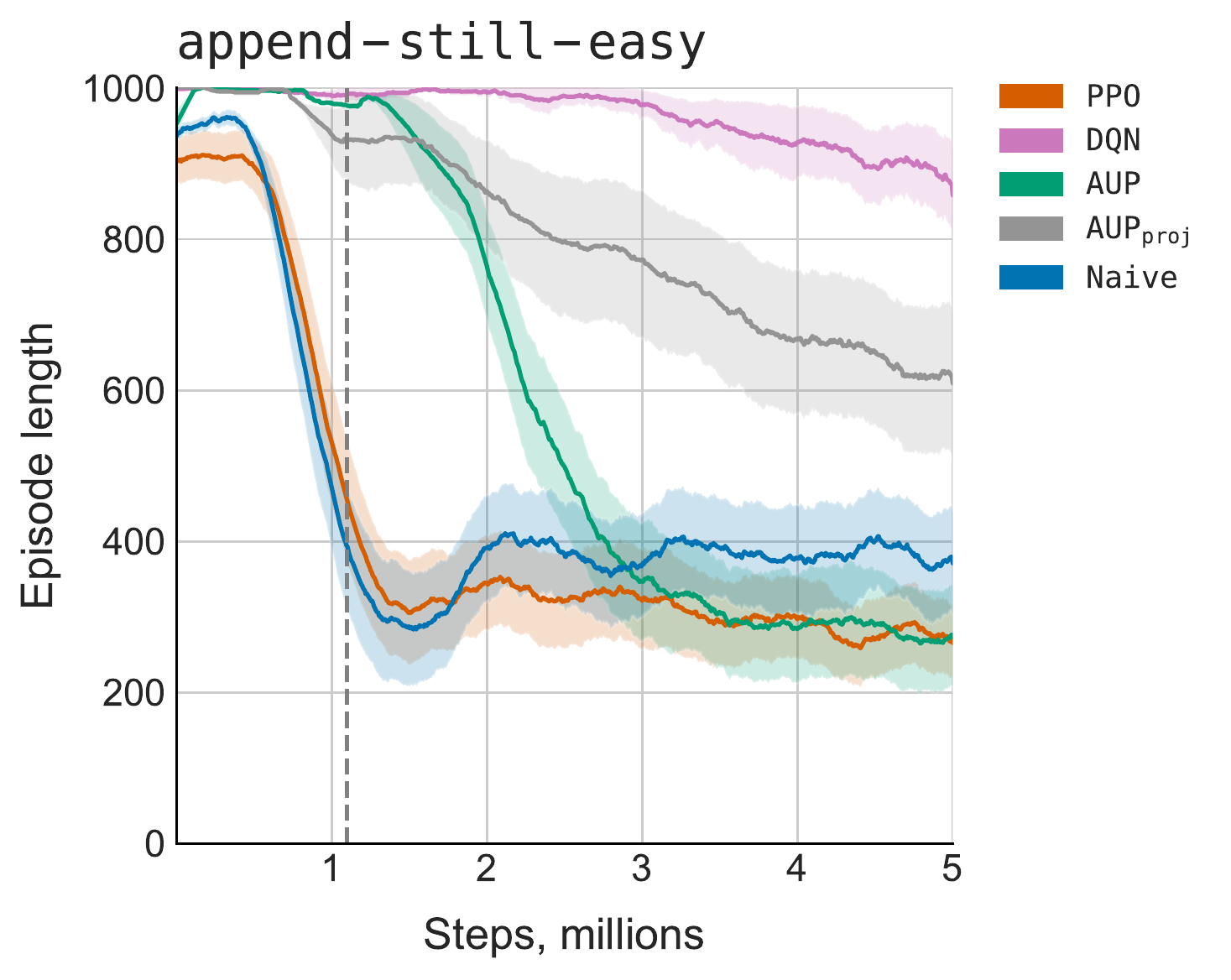}
}
\subfloat[][]{
    \includegraphics[height=2.1in]{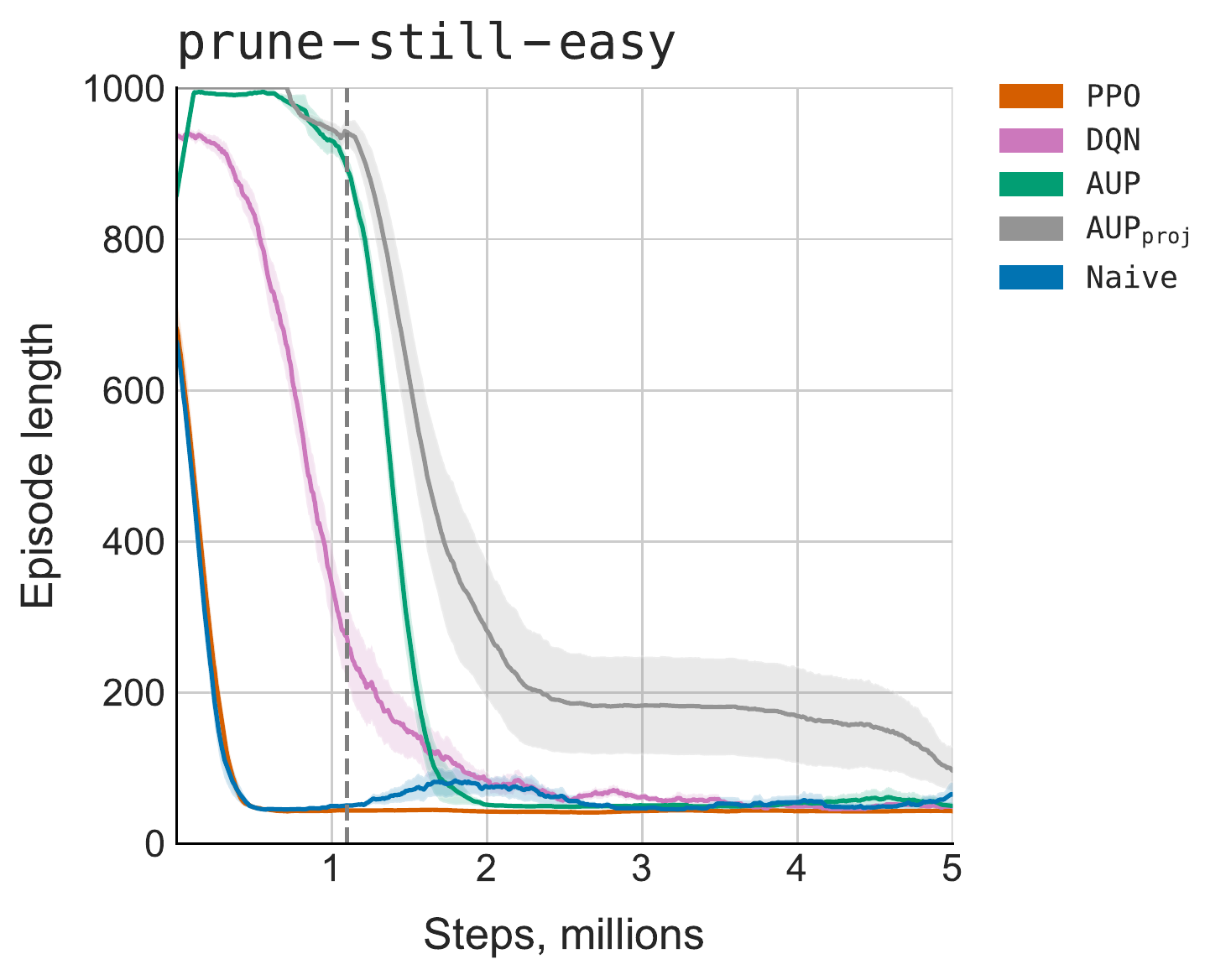}
}

\subfloat[][]{
    \includegraphics[height=2.1in]{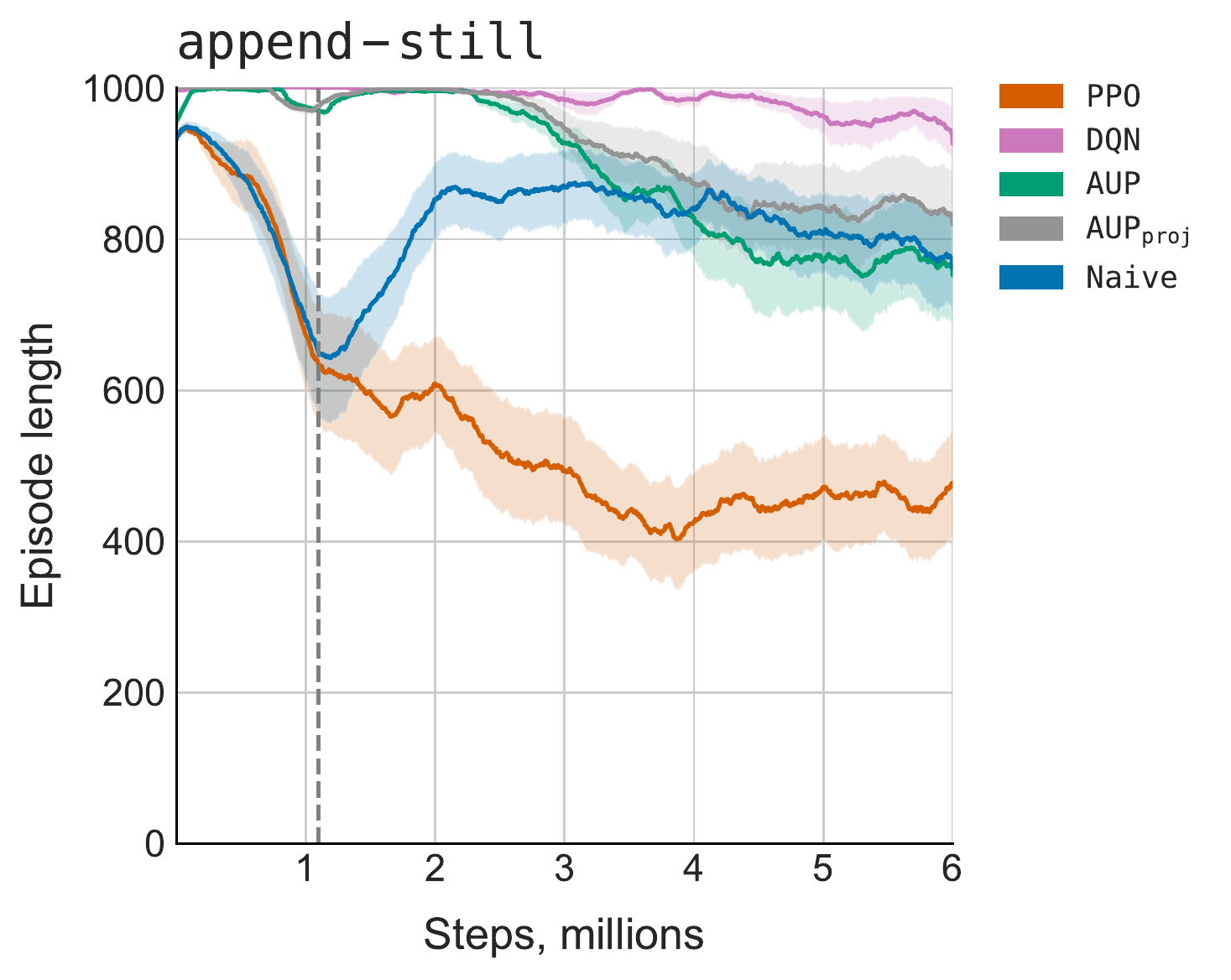}
}
\subfloat[][]{
    \includegraphics[height=2.1in]{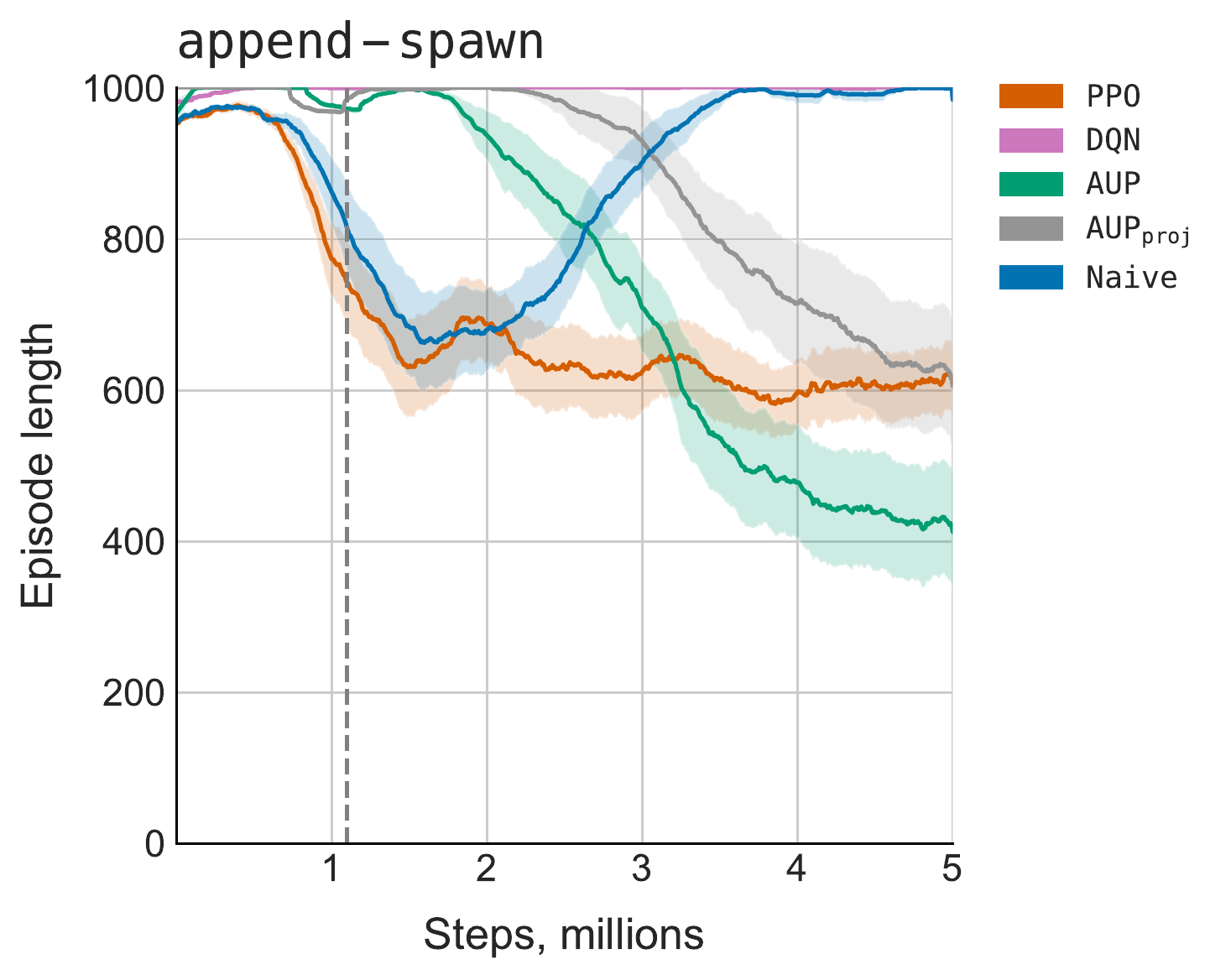}
}
\end{figure}

\begin{figure}[ht!]
\centering
\caption{Reward curves for auxiliary reward functions with a $Z$-dimensional latent space. Shaded regions represent $\pm 1$ standard deviation. Auxiliary reward is not comparable across trials, so learning is expressed by the slope of the curves.} \label{fig:aux-reward}
\includegraphics[height=2.5in]{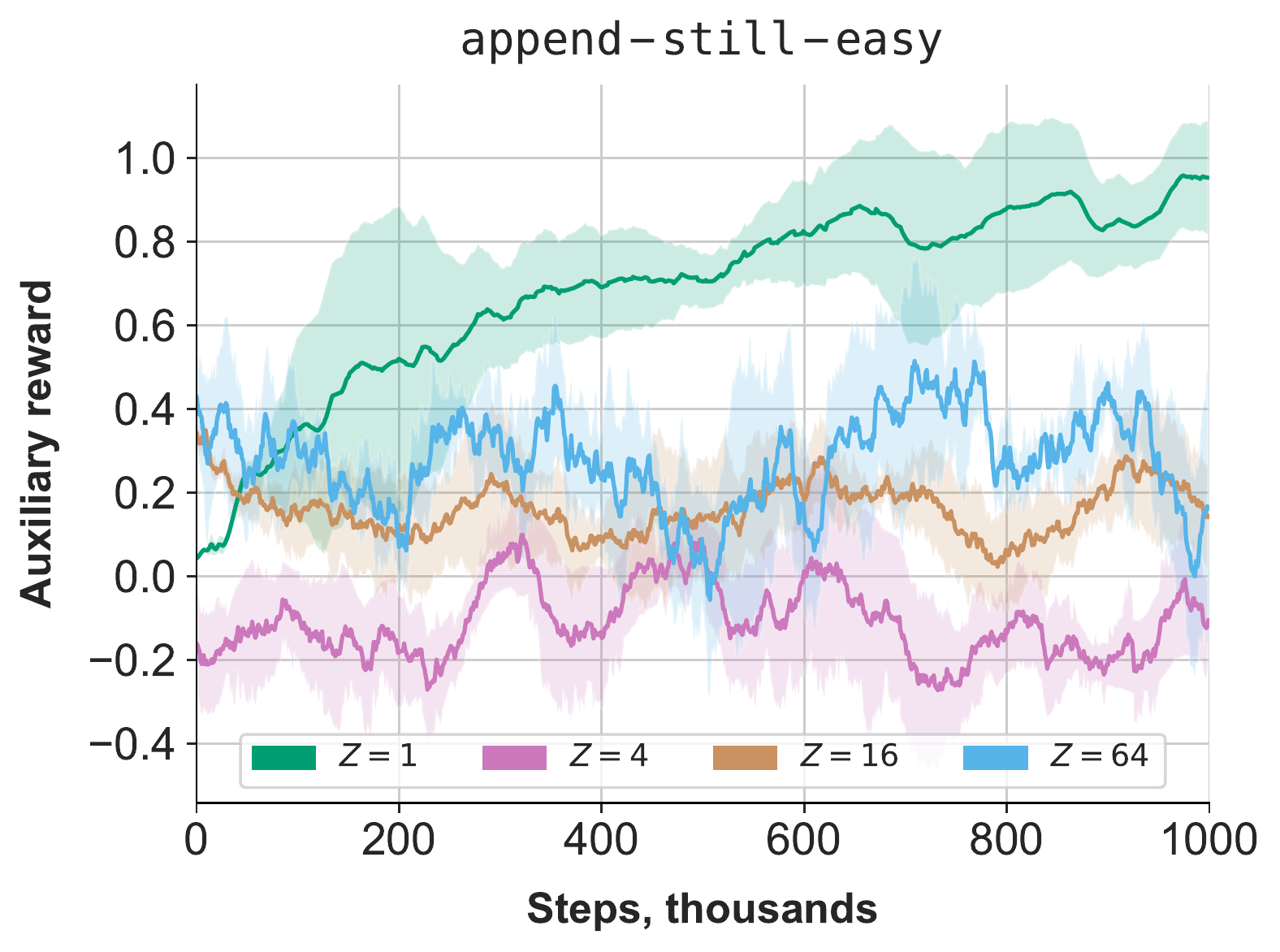}
\end{figure}

\begin{figure}[h]
\caption{Smoothed learning curves for individual {\aup} seeds. {\aup} begins training on the $R_\textsc{aup}$ reward signal at step 1.1M, marked by a dotted vertical line.}\label{fig:individual} 

\captionsetup[subfigure]{labelformat=empty, captionskip=-50pt}
\subfloat[][]{
    \includegraphics[height=2.05in]{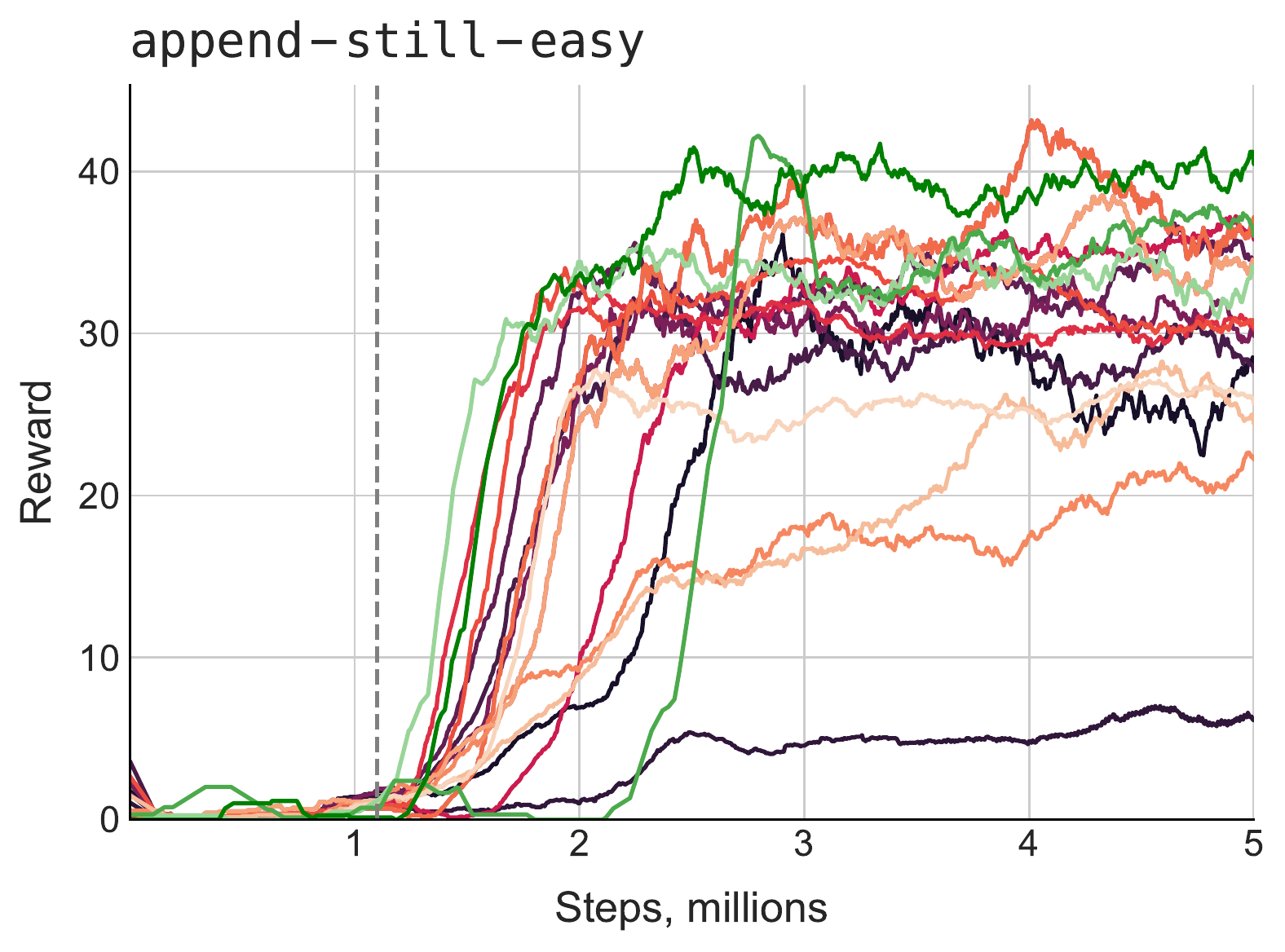}
}
\subfloat[][]{
    \includegraphics[height=2.05in]{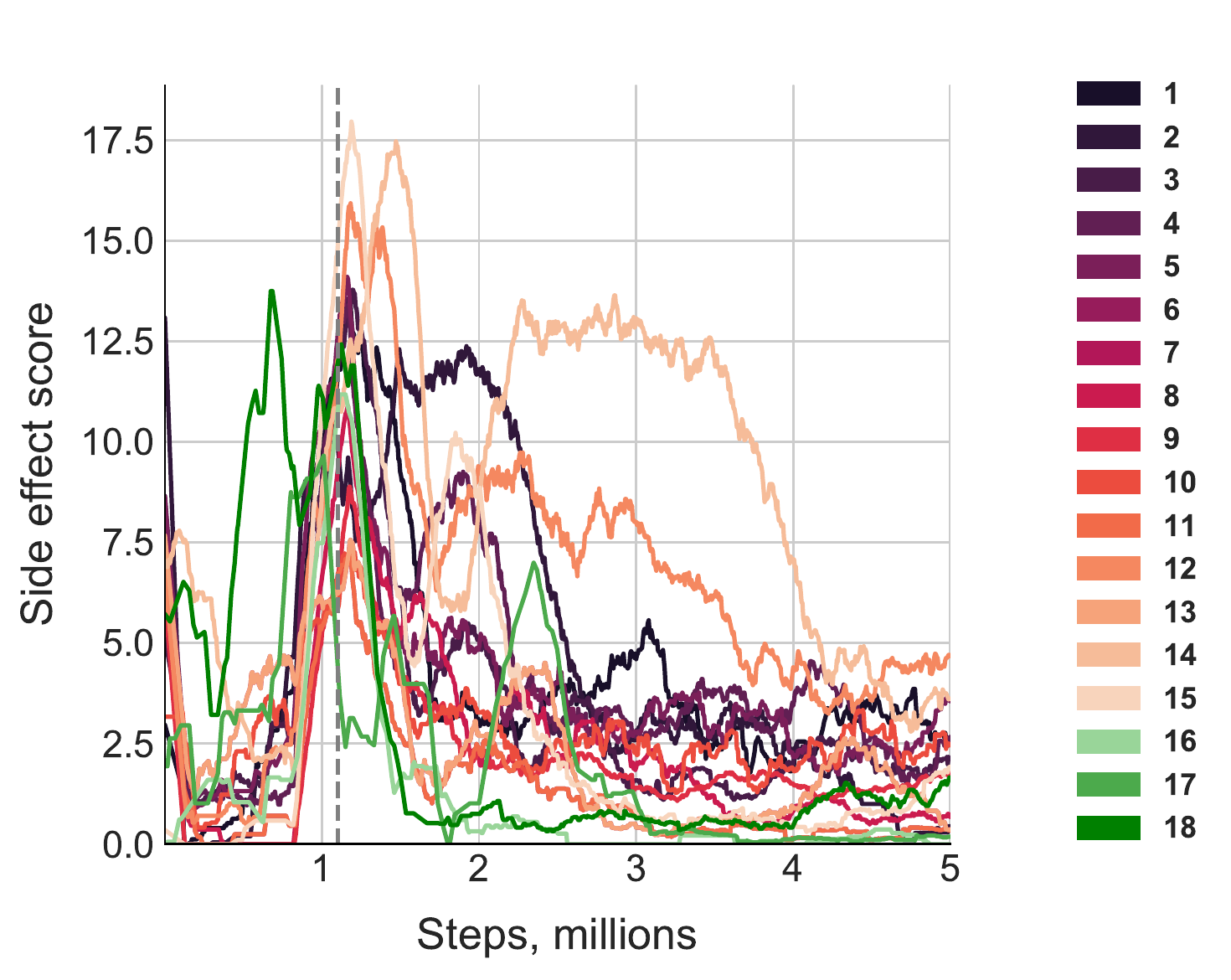}
}
\vspace{-8pt}
\subfloat[][]{
    \includegraphics[height=2.05in]{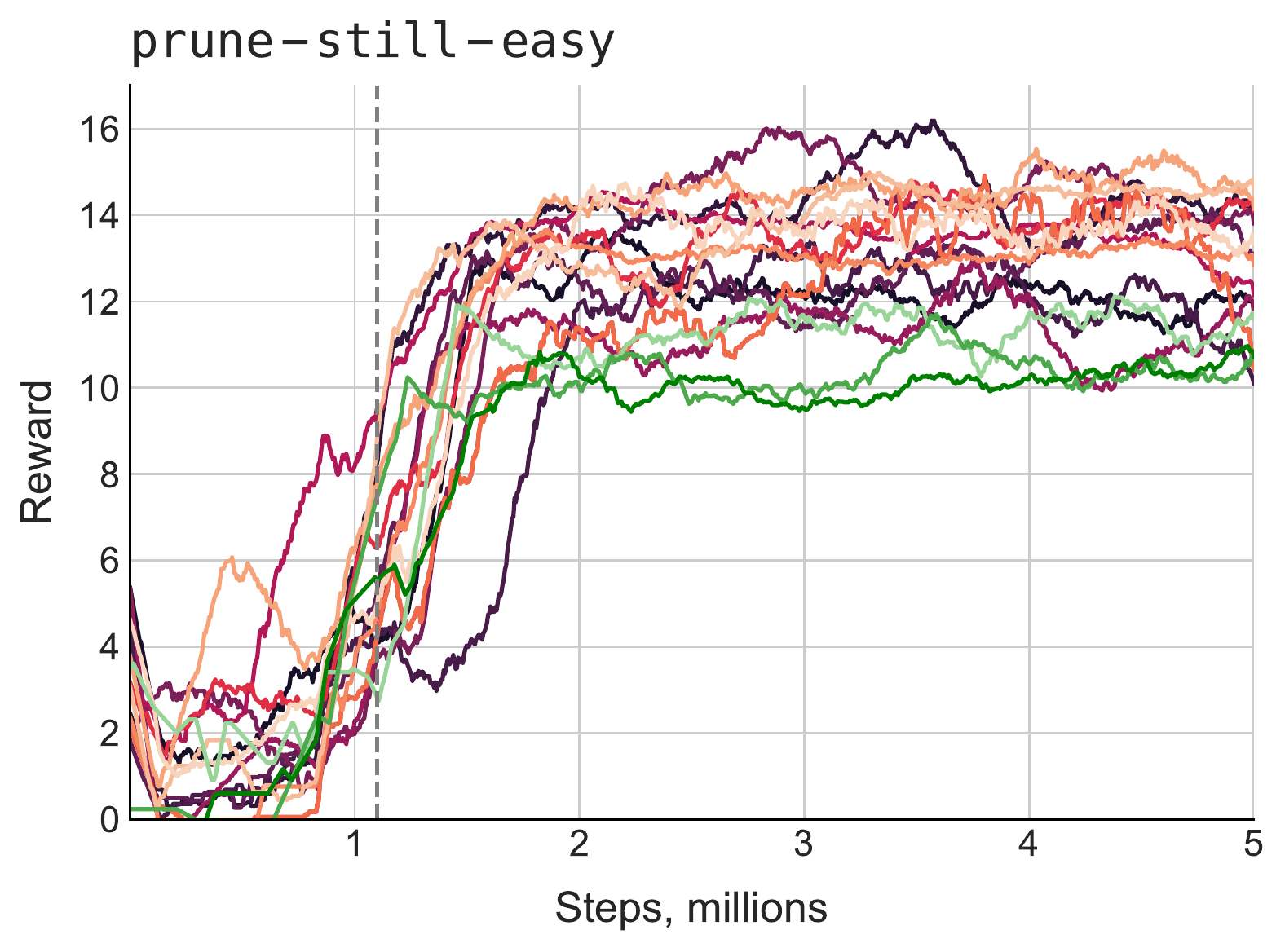}
}
\subfloat[][]{
    \includegraphics[height=2.05in]{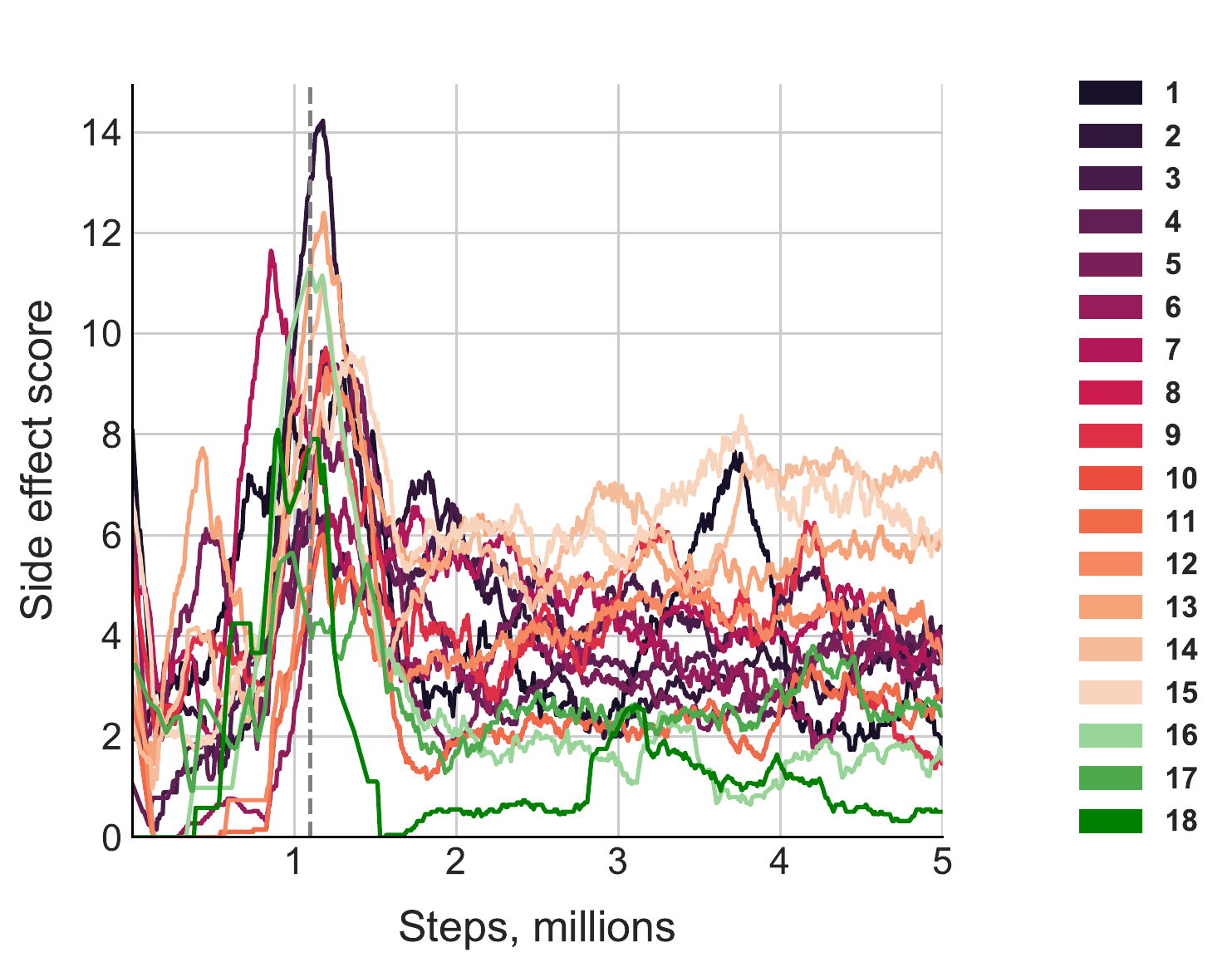}
}
\vspace{-8pt}

\subfloat[][]{
    \includegraphics[height=2.05in]{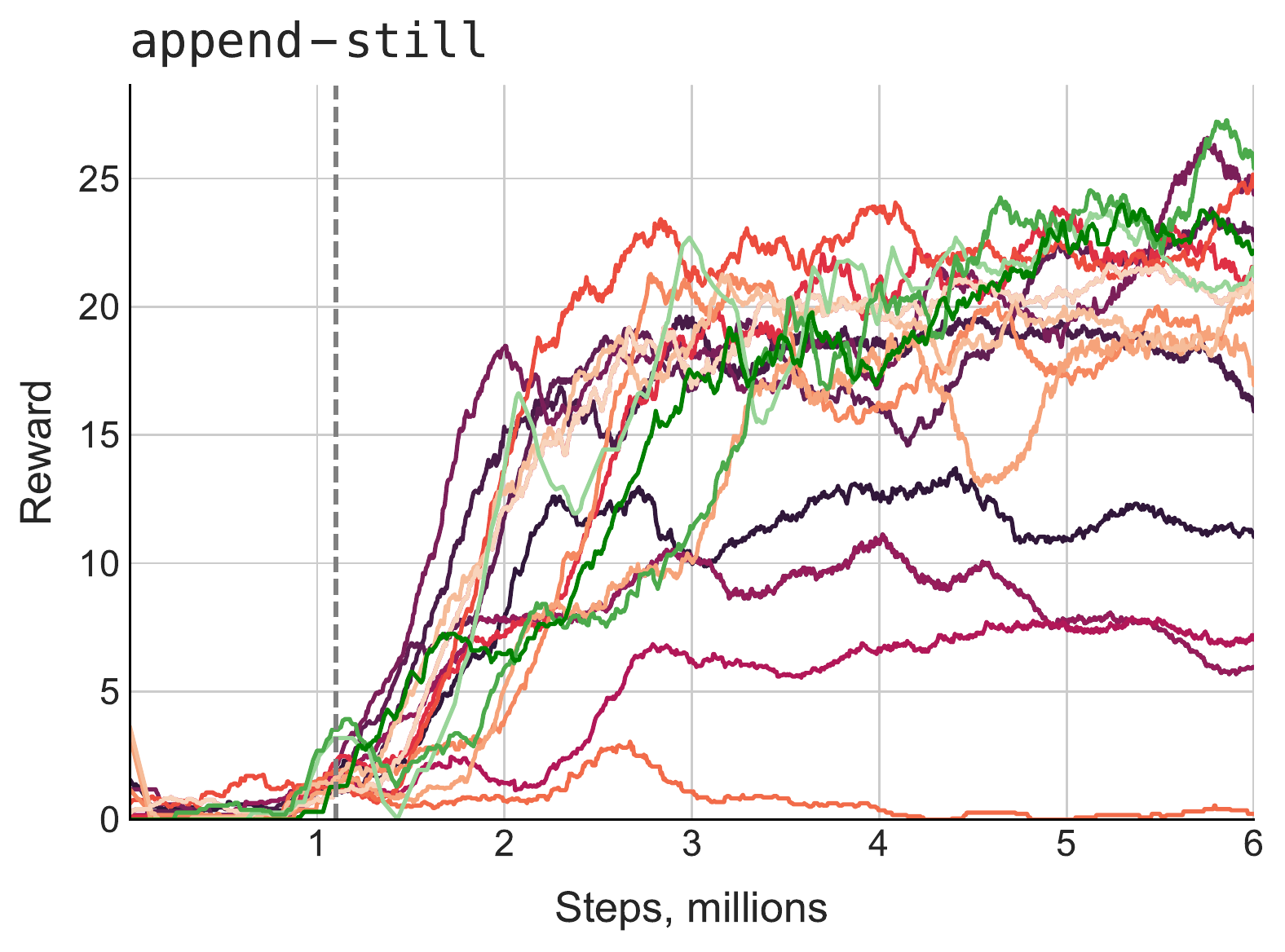}
}
\subfloat[][]{
    \includegraphics[height=2.05in]{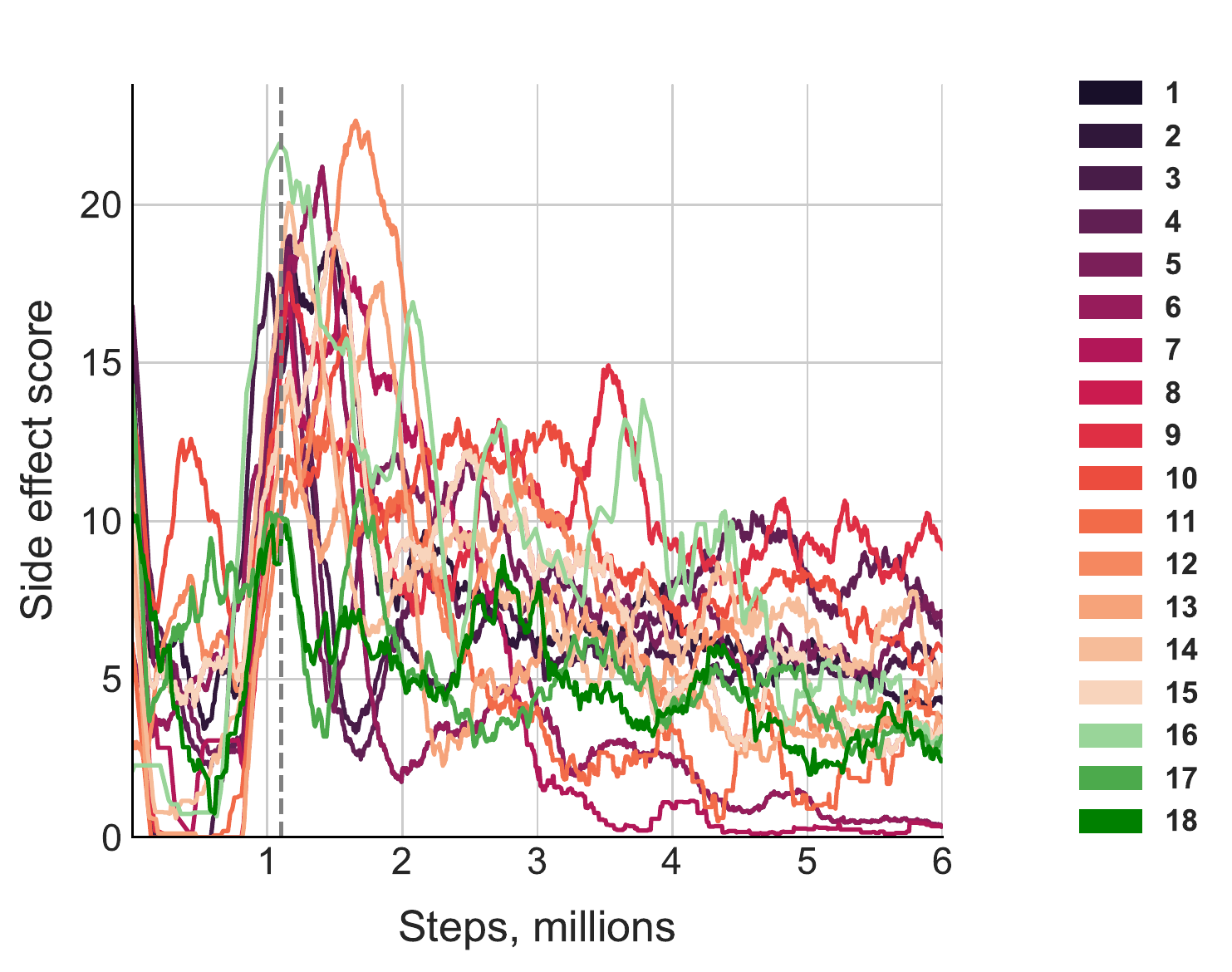}
}
\vspace{-8pt}
\subfloat[][]{
    \includegraphics[height=2.05in]{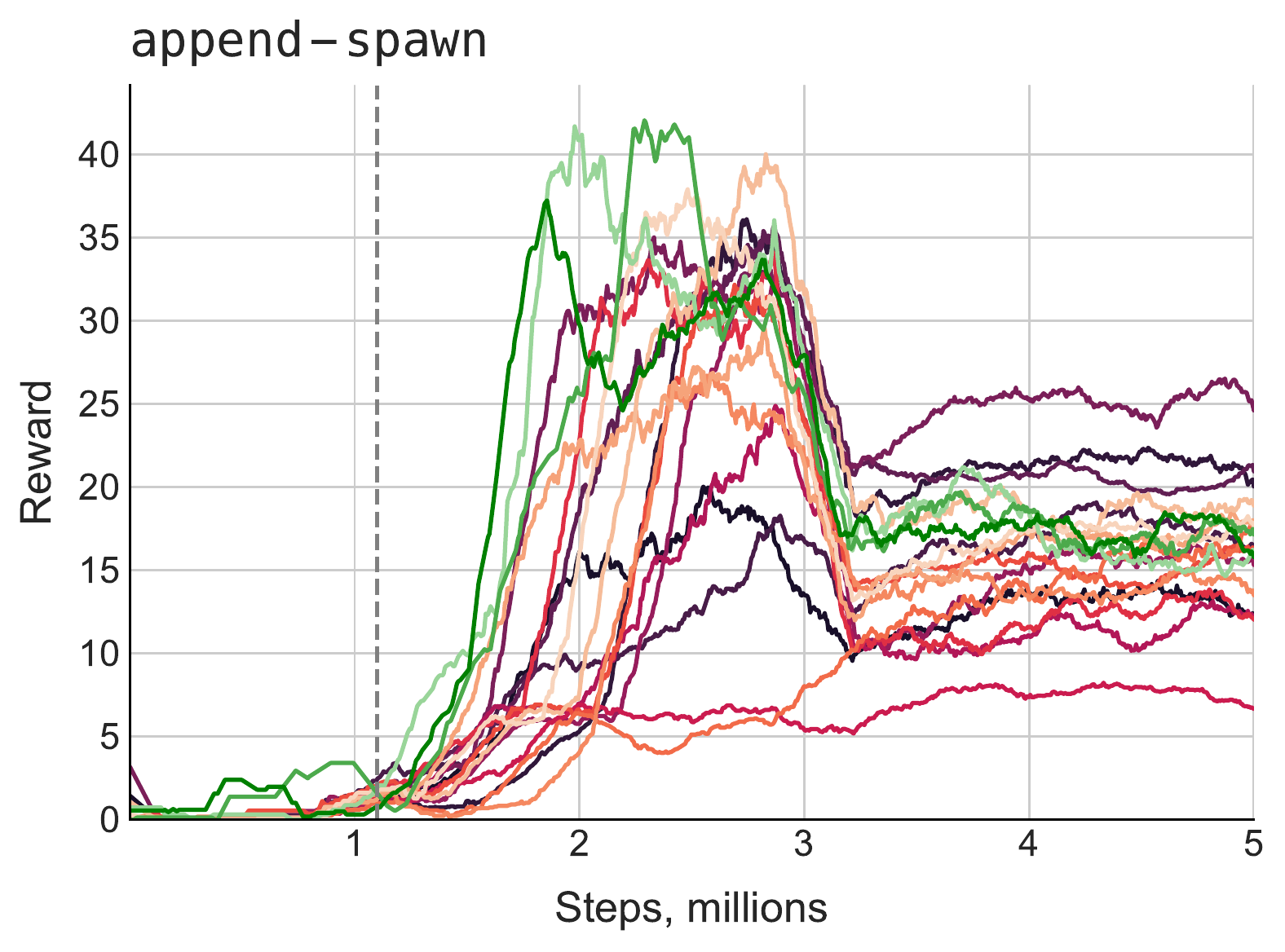}
}
\subfloat[][]{
    \includegraphics[height=2.05in]{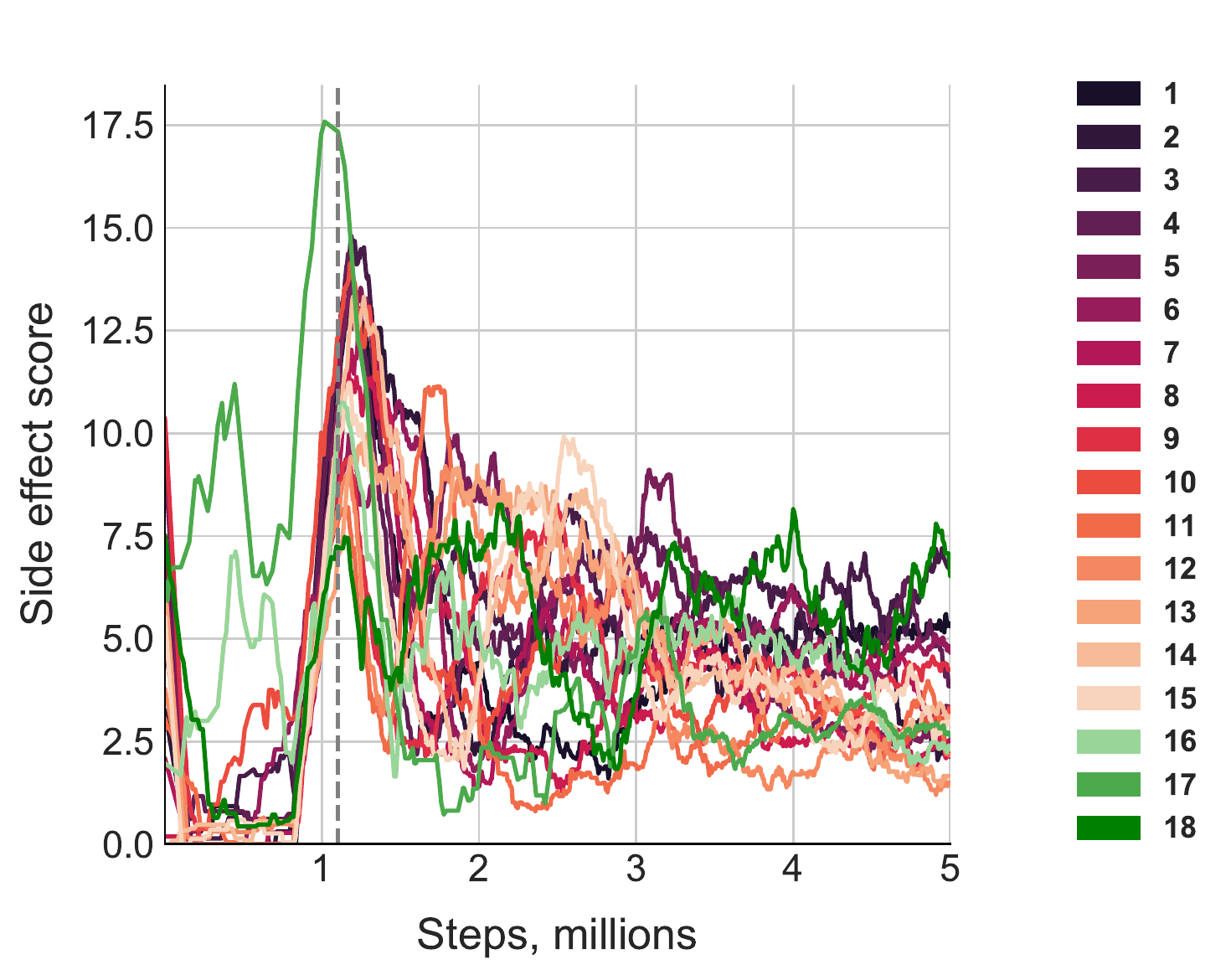}
}
\end{figure}

\begin{figure}[h]
\centering 
\caption{Smoothed episode length curves for individual {\aup} seeds.}\label{fig:individ-length}

\captionsetup[subfigure]{labelformat=empty, captionskip=-50pt}
\subfloat[][]{
    \includegraphics[height=2in]{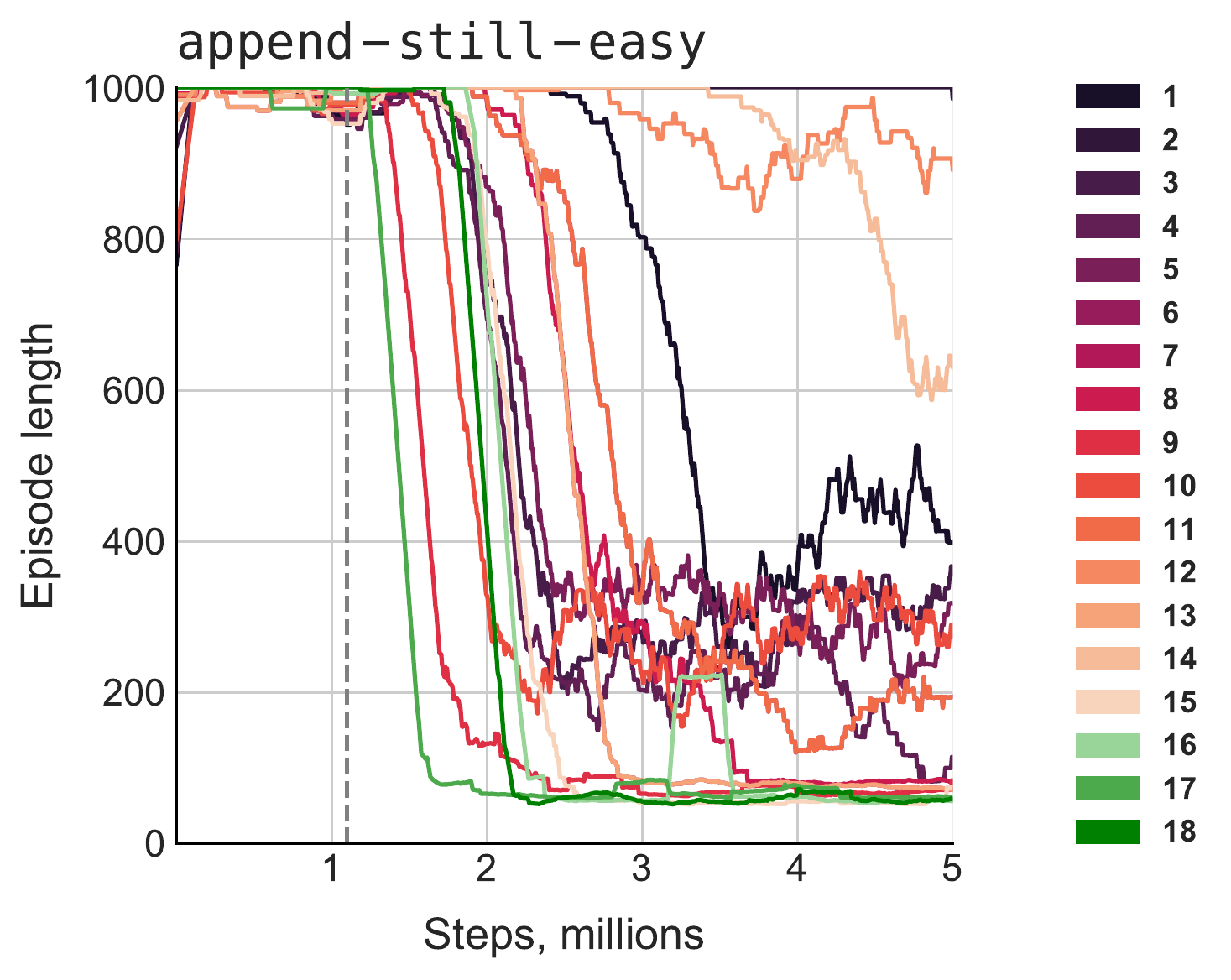}
}
\subfloat[][]{
    \includegraphics[height=2in]{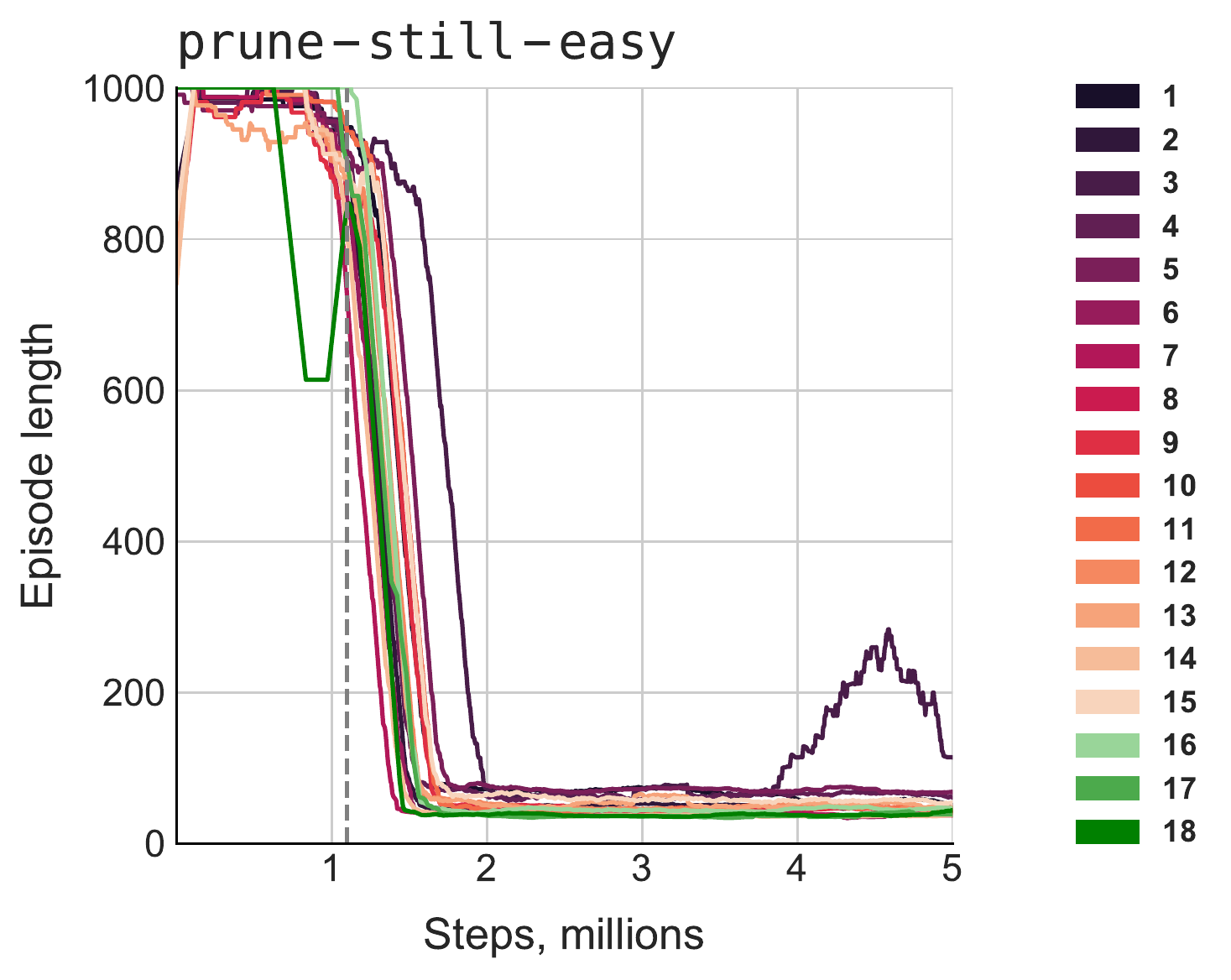}
}

\subfloat[][]{
    \includegraphics[height=2in]{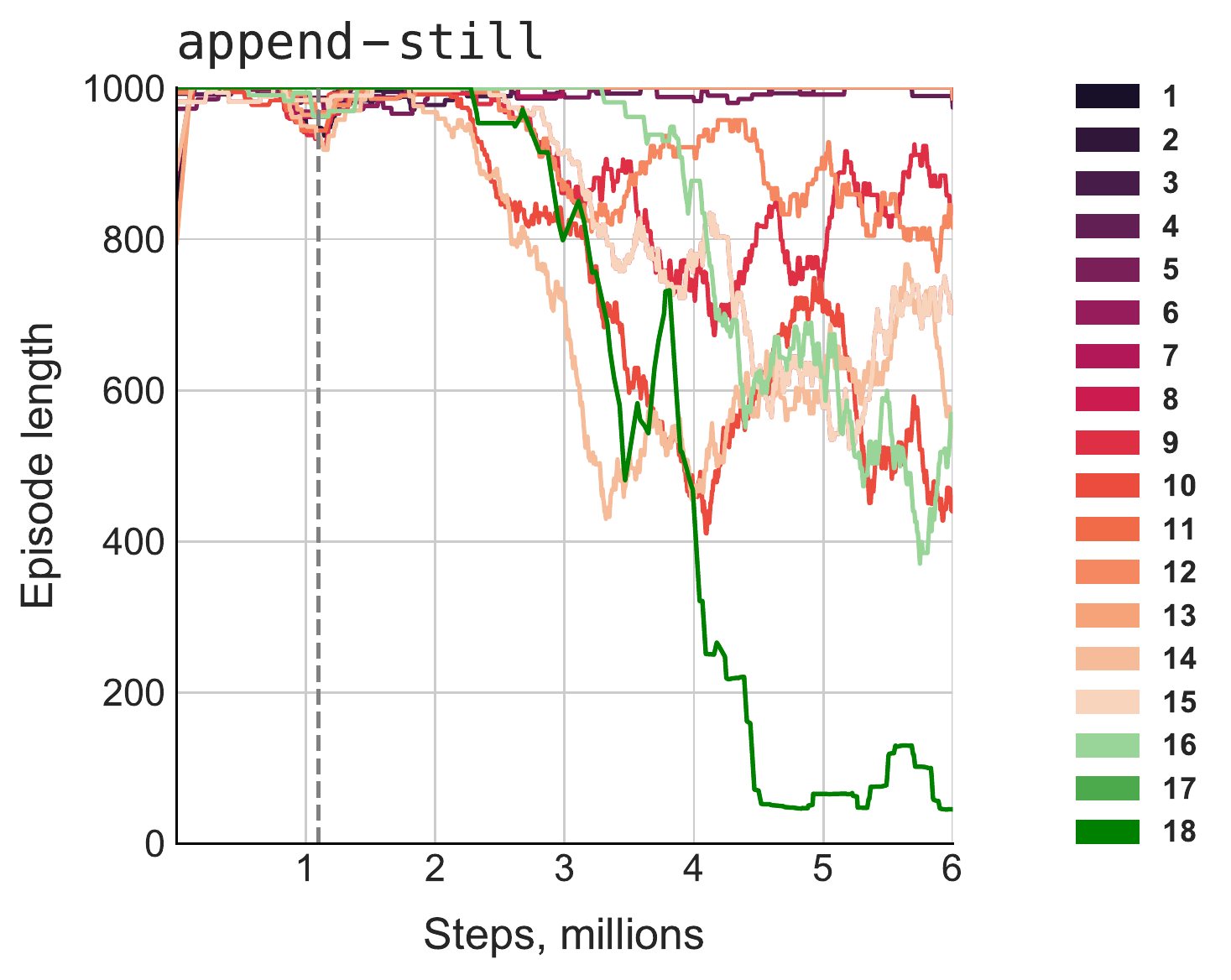}
}
\subfloat[][]{
    \includegraphics[height=2in]{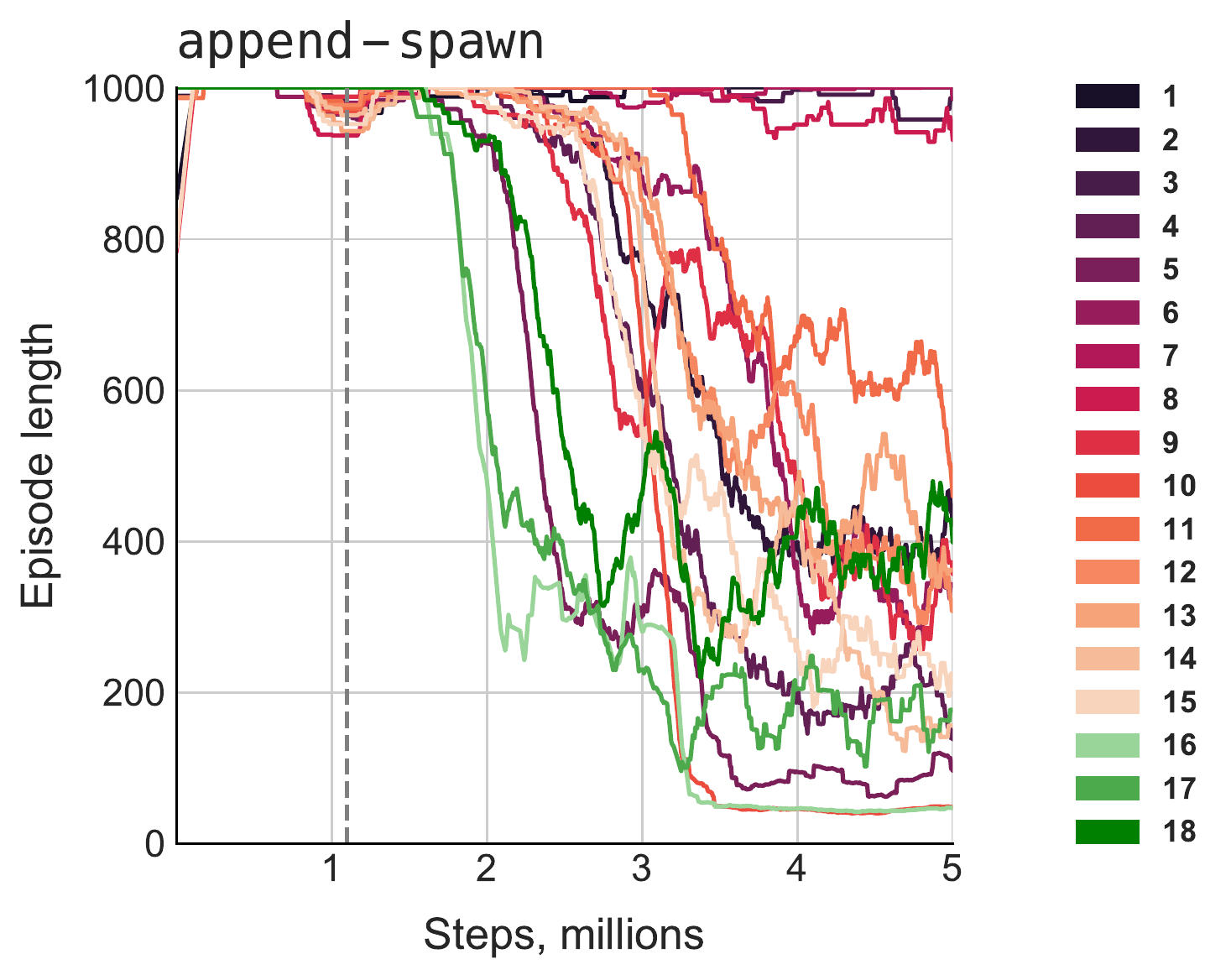}
}
\end{figure}

\begin{figure}[h]
\caption{Smoothed learning curves for {\aup} on its four curricula. {\aup} begins training on the $R_\textsc{aup}$ reward signal at step 1.1M, marked by a dotted vertical line.}\label{fig:batch} 

\captionsetup[subfigure]{labelformat=empty, captionskip=-50pt}
\subfloat[][]{
    \includegraphics[height=2.05in]{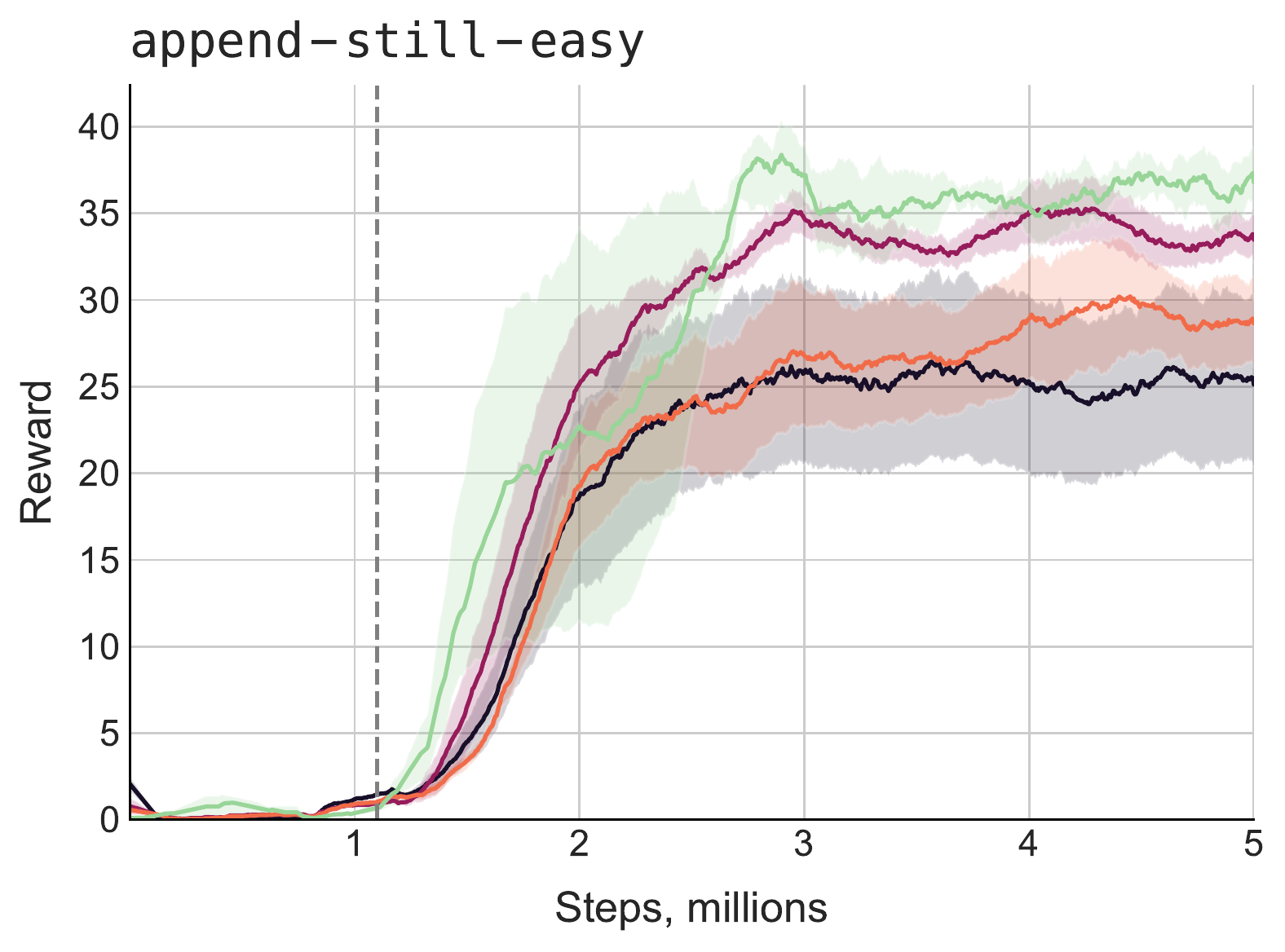}
}
\subfloat[][]{
    \includegraphics[height=2.05in]{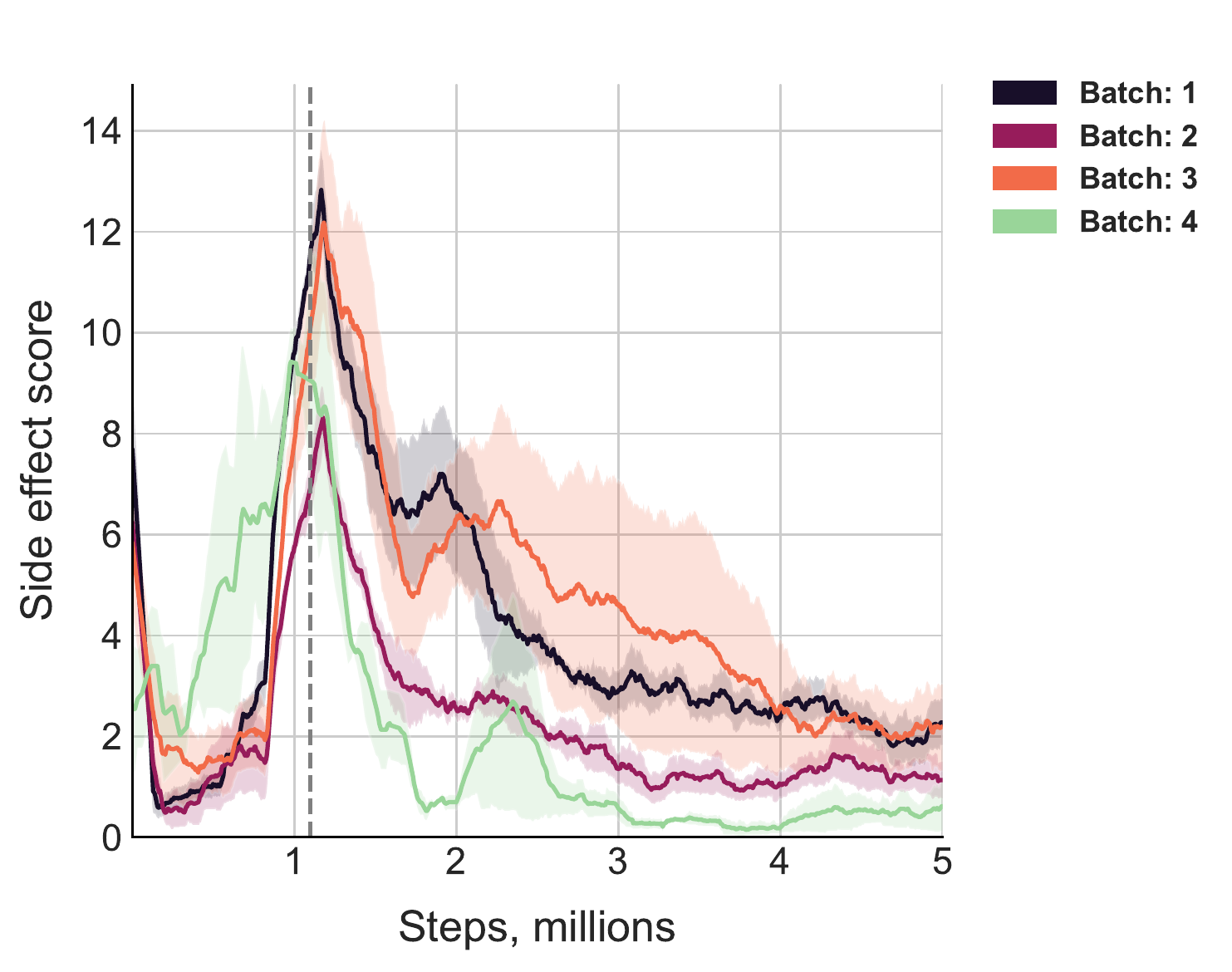}
}
\vspace{-8pt}
\subfloat[][]{
    \includegraphics[height=2.05in]{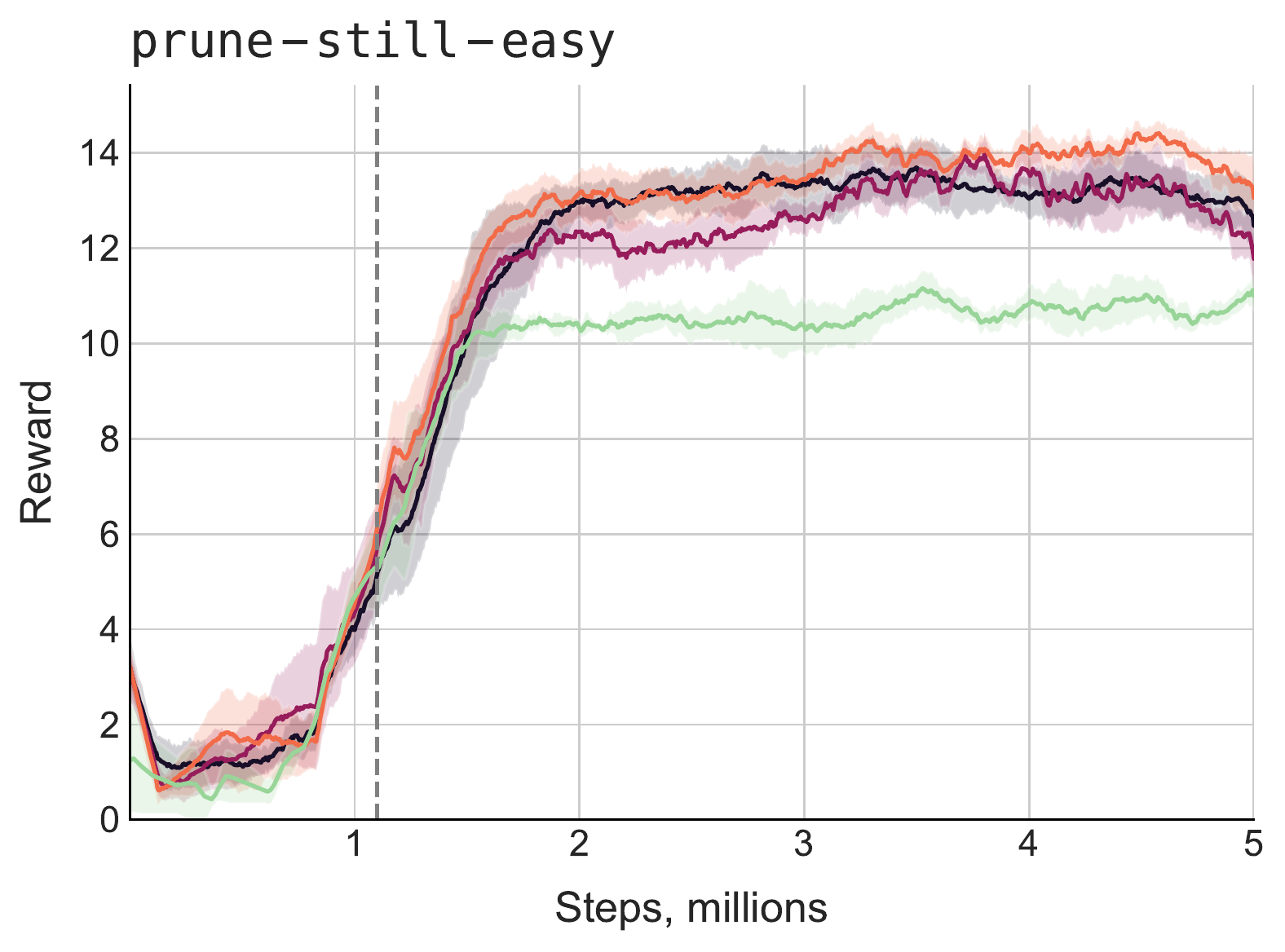}
}
\subfloat[][]{
    \includegraphics[height=2.05in]{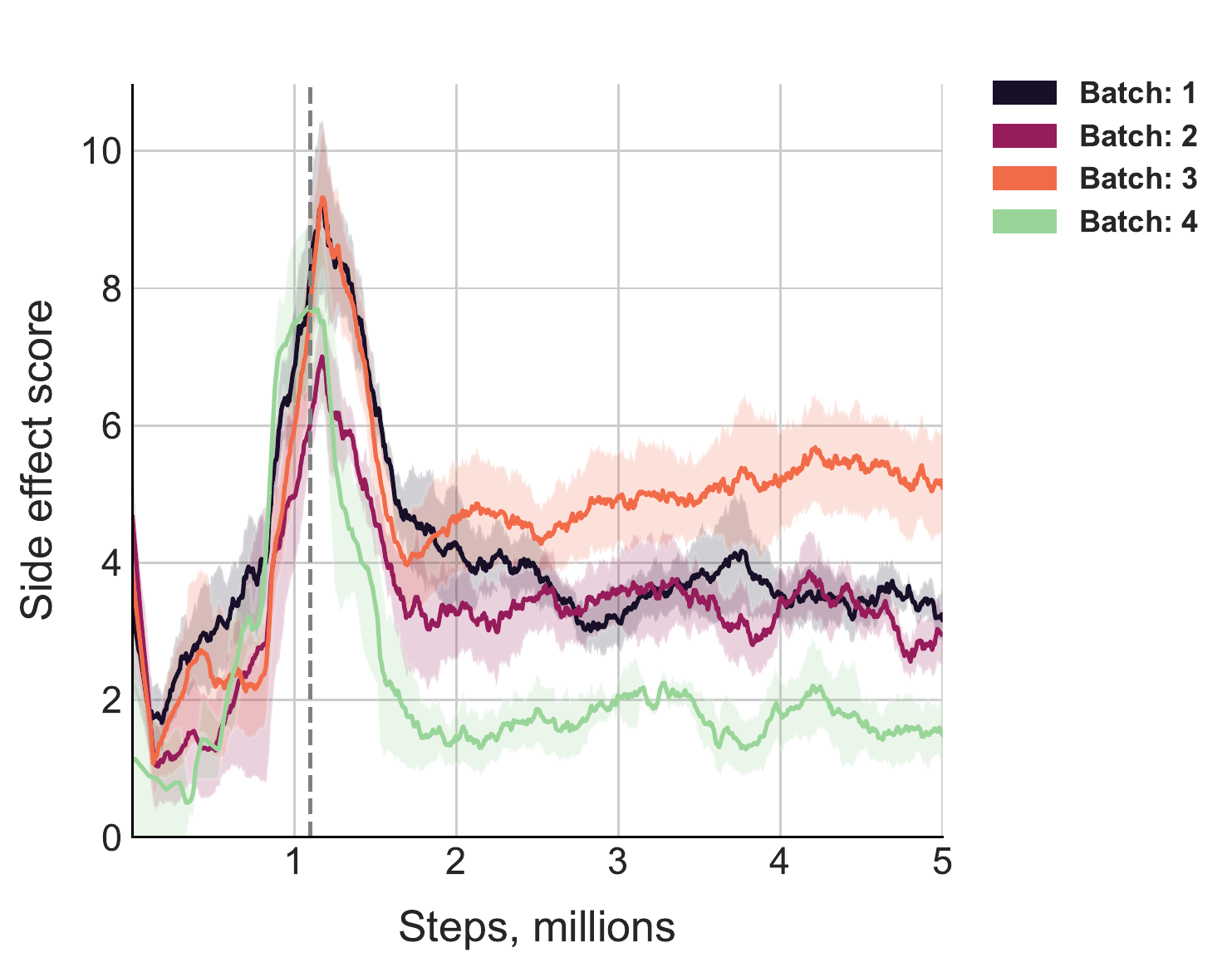}
}
\vspace{-8pt}

\subfloat[][]{
    \includegraphics[height=2.05in]{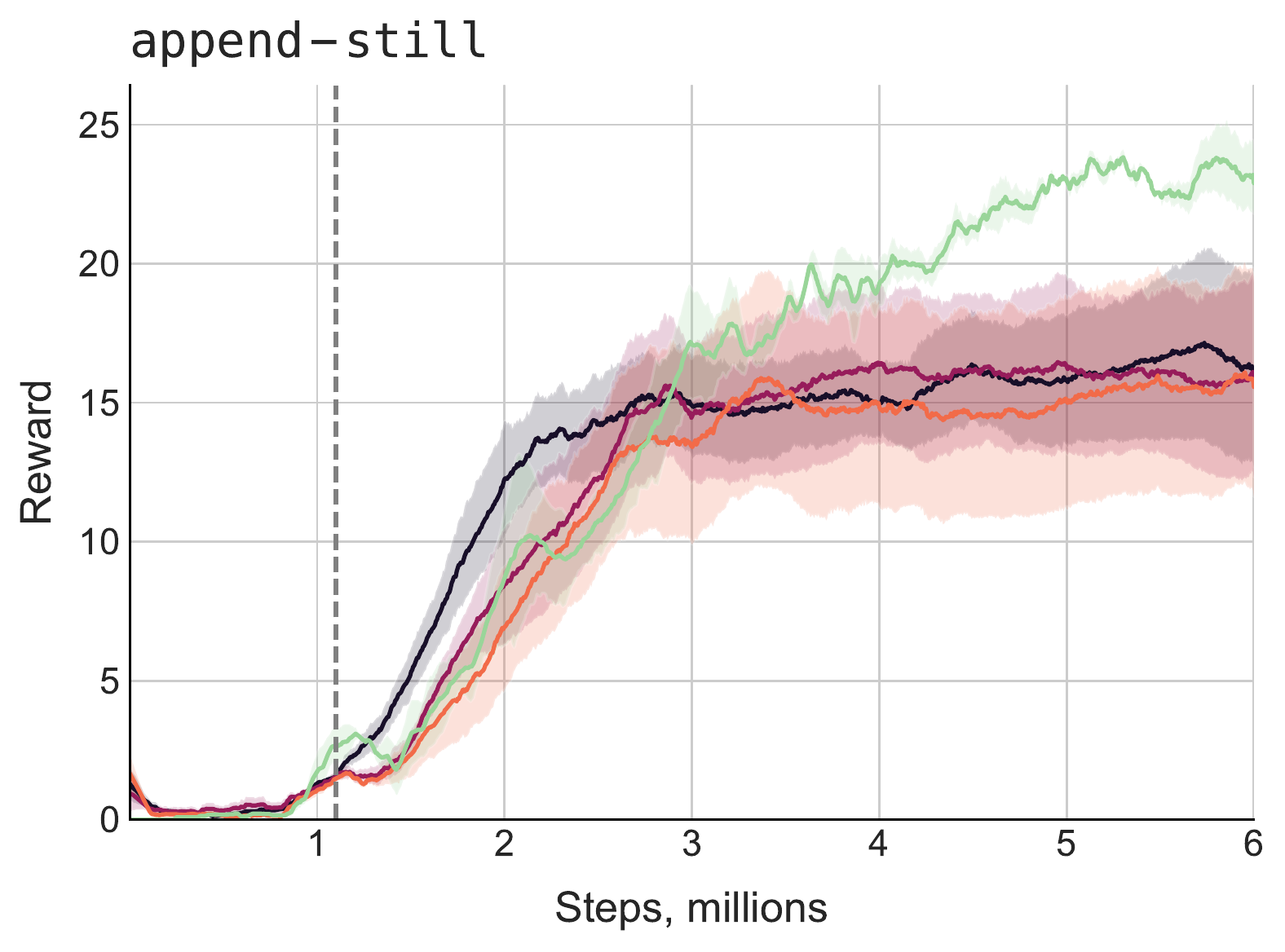}
}
\subfloat[][]{
    \includegraphics[height=2.05in]{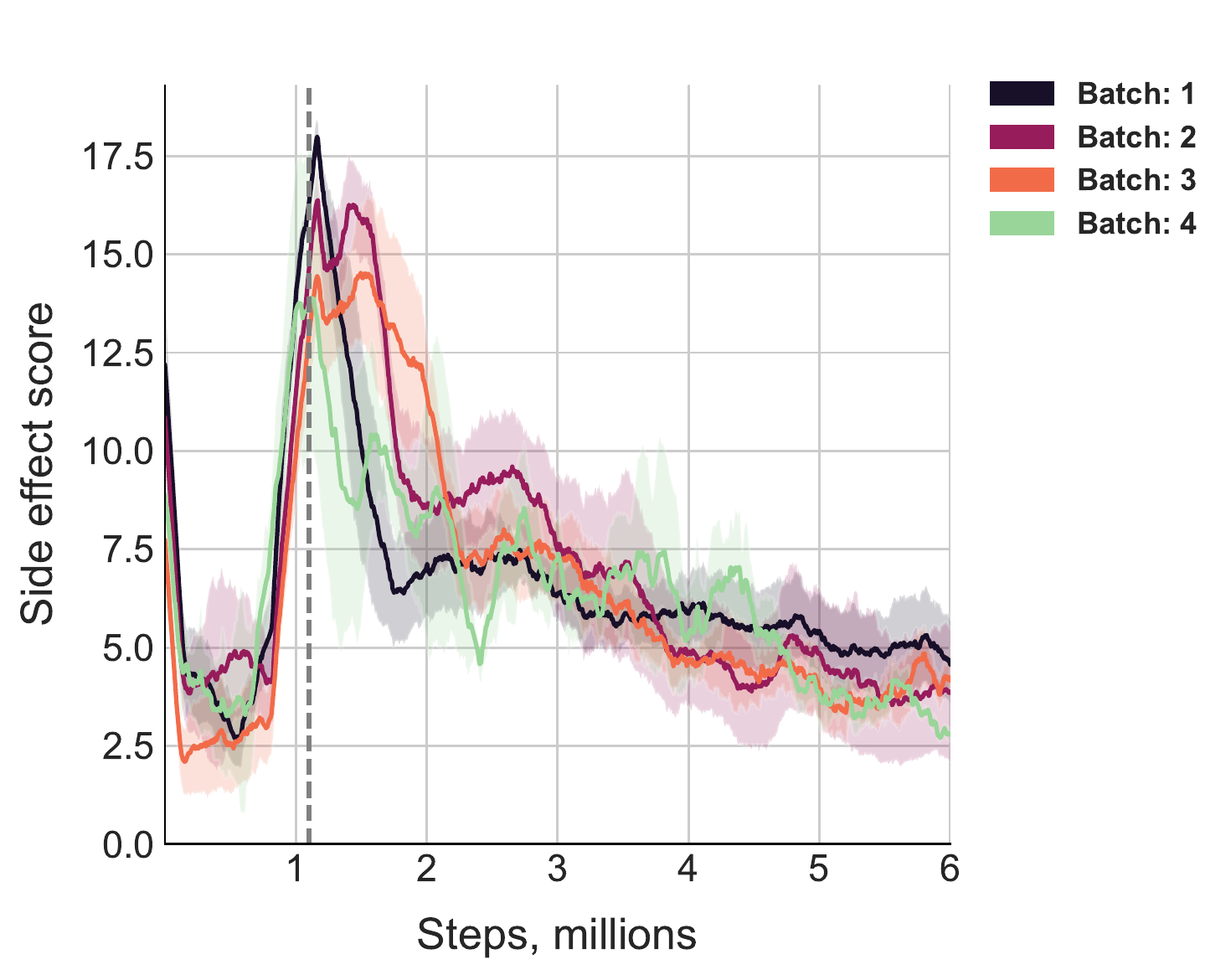}
}
\vspace{-8pt}
\subfloat[][]{
    \includegraphics[height=2.05in]{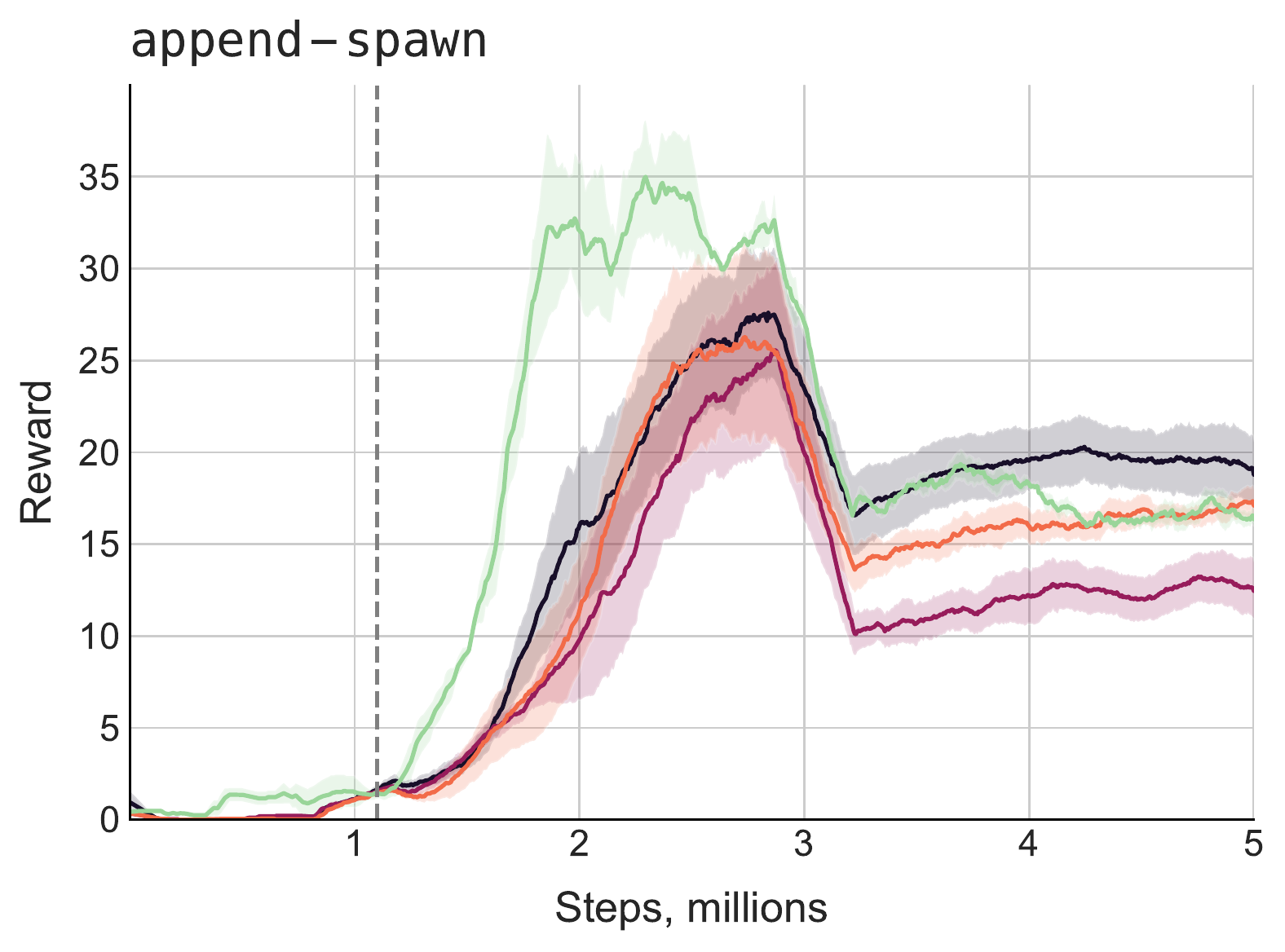}
}
\subfloat[][]{
    \includegraphics[height=2.05in]{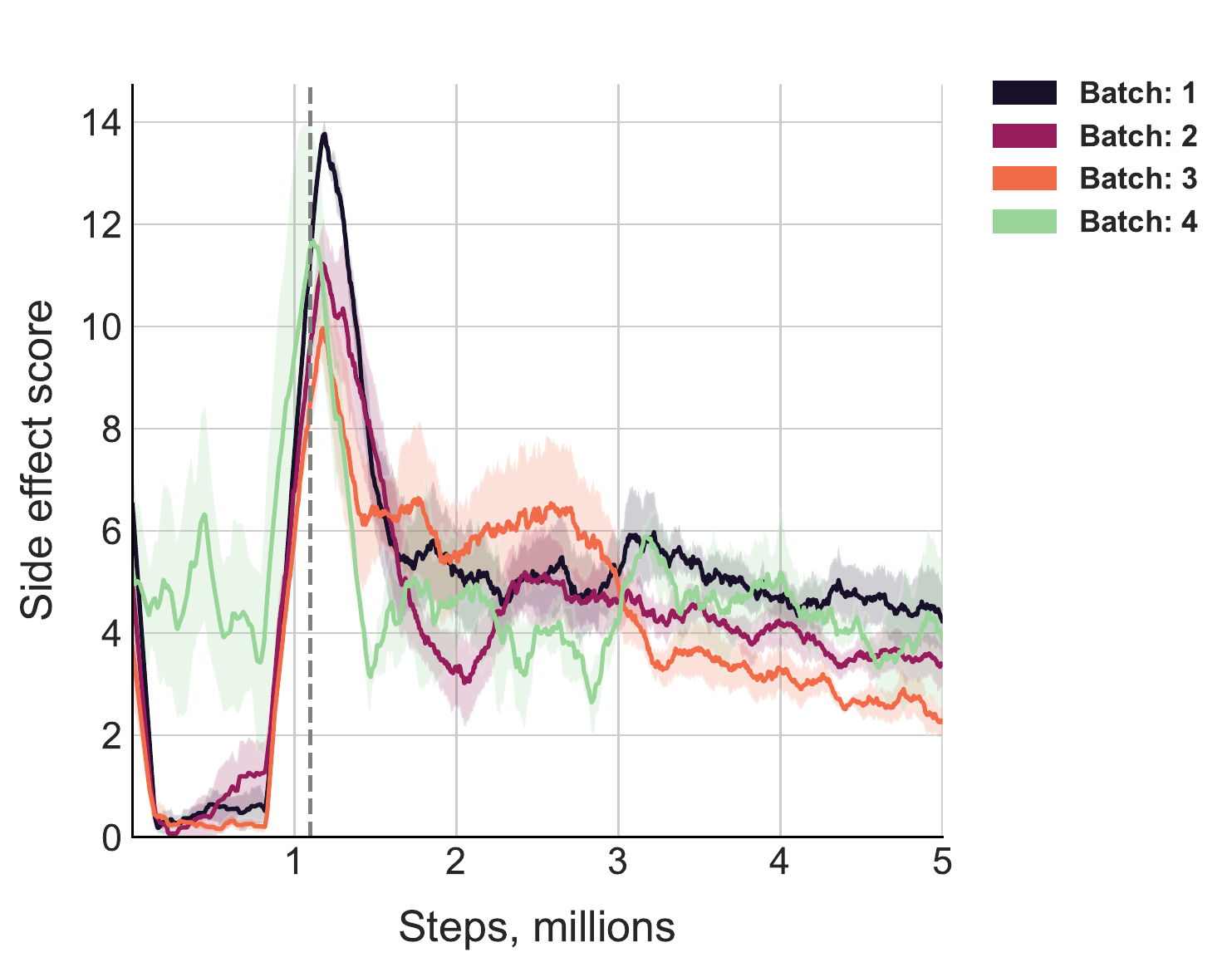}
}
\end{figure}

\begin{figure}[h]
\centering 
\caption{Smoothed episode length curves for {\aup} on each of the four curricula.}\label{fig:batch-length}

\captionsetup[subfigure]{labelformat=empty, captionskip=-50pt}
\subfloat[][]{
    \includegraphics[height=2in]{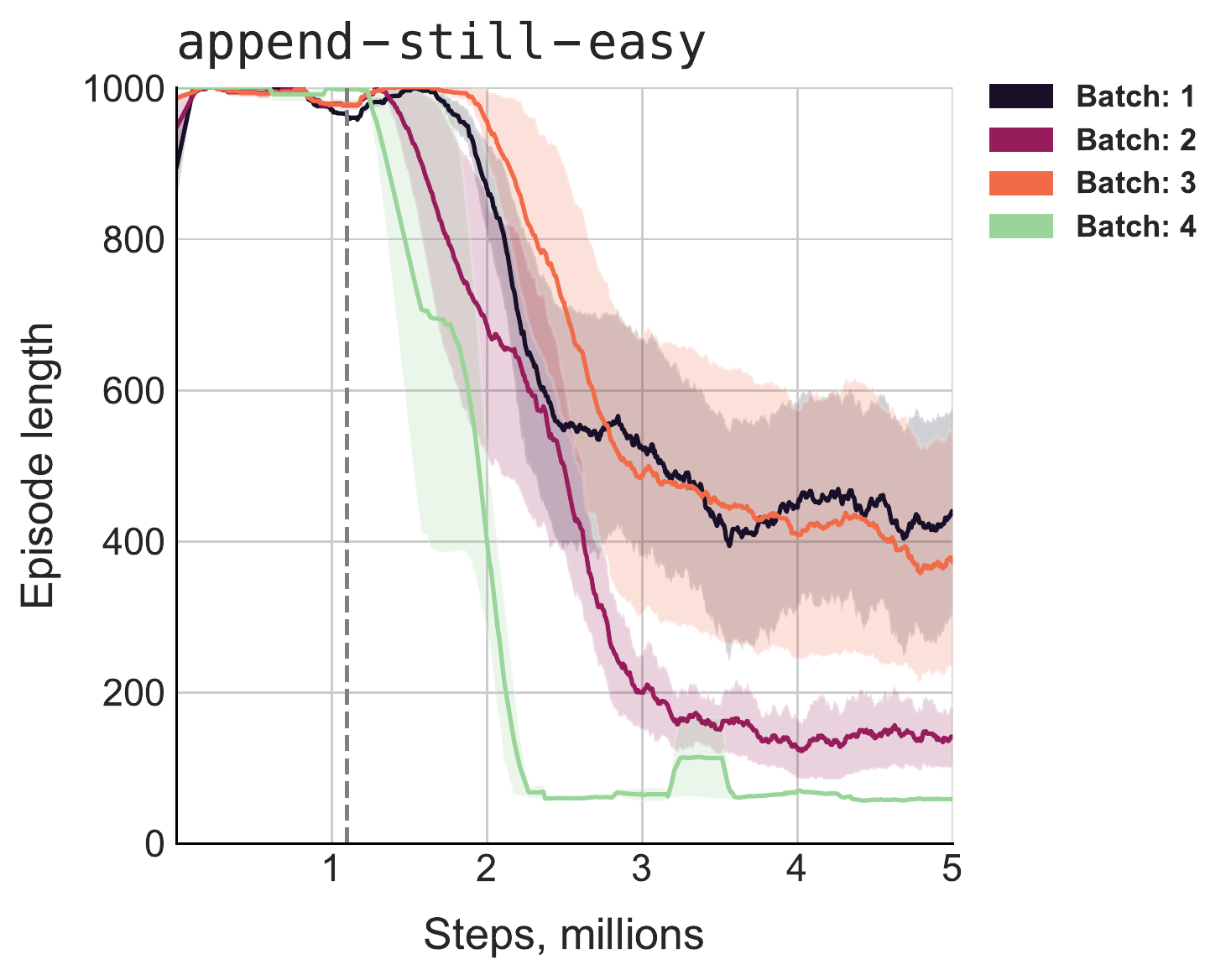}
}
\subfloat[][]{
    \includegraphics[height=2in]{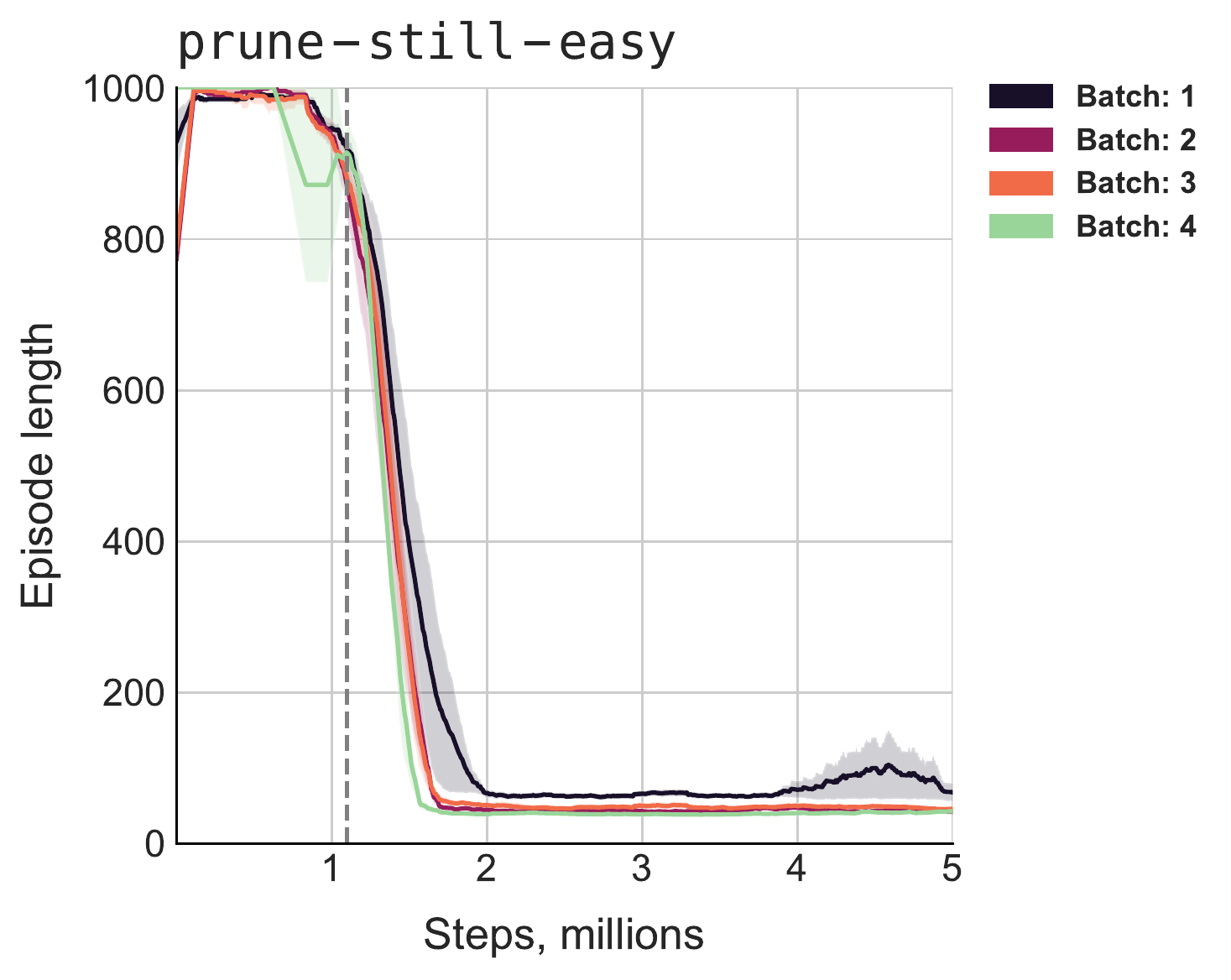}
}

\subfloat[][]{
    \includegraphics[height=2in]{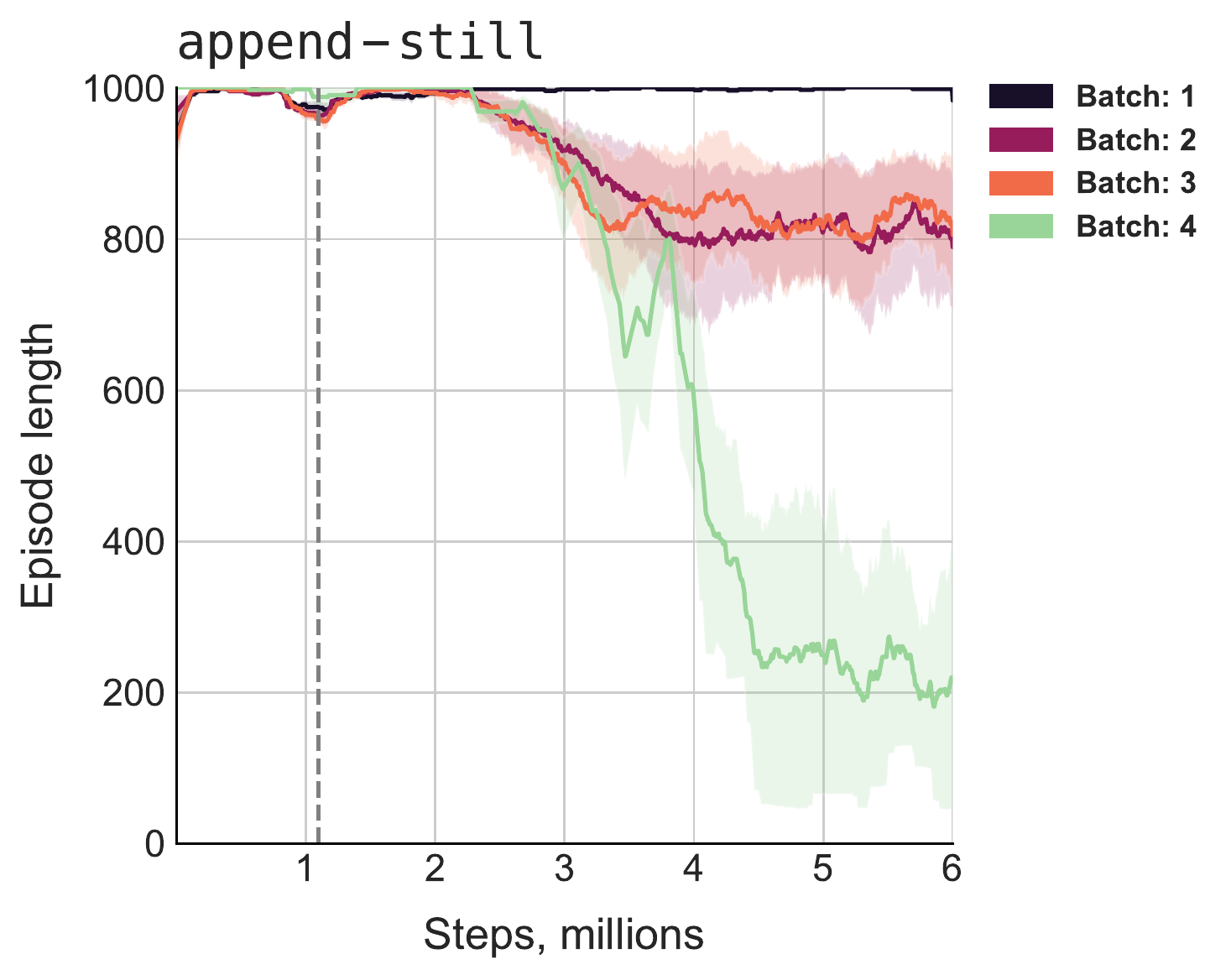}
}
\subfloat[][]{
    \includegraphics[height=2in]{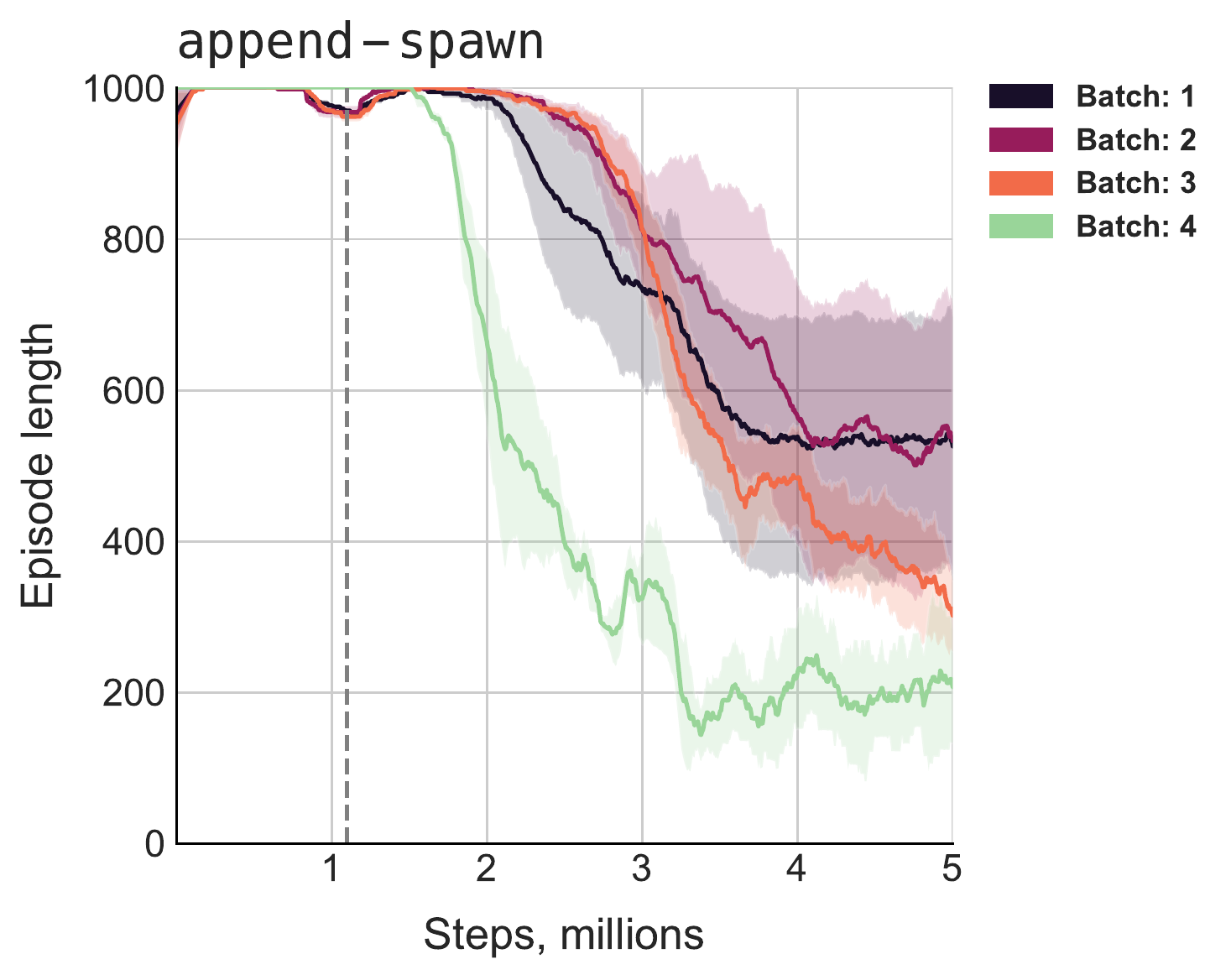}
}
\end{figure}
\end{document}